\documentclass[twoside]{article}
\pdfoutput=1
\usepackage{amsmath,amsfonts,bm}
\usepackage{amsthm}

\def\eqref#1{equation~\ref{#1}}

\def\1{\bm{1}}

\def\vp{{\bm{p}}}

\def\vs{{\bm{s}}}
\def\vt{{\bm{t}}}
\def\vu{{\bm{u}}}

\def\vx{{\bm{x}}}

\def\vz{{\bm{z}}}

\def\mI{{\bm{I}}}

\DeclareMathAlphabet{\mathsfit}{\encodingdefault}{\sfdefault}{m}{sl}
\SetMathAlphabet{\mathsfit}{bold}{\encodingdefault}{\sfdefault}{bx}{n}

\DeclareMathOperator*{\argmax}{arg\,max}

\newtheorem{defn}{Definition}

\newtheorem*{prop*}{Proposition} 
\newtheorem{thm}{Theorem}
\newtheorem*{thm*}{Theorem}

\newtheorem{rmk}{Remark}

\newtheorem{lemma}{Lemma}

\newtheorem*{example*}{Example}

\newcommand\norm[1]{\lVert#1\rVert}

\usepackage{todonotes}

\usepackage{booktabs}
\usepackage{amsmath}
\usepackage{algorithm}
\usepackage{algpseudocode}

\usepackage[utf8]{inputenc}

\renewcommand{\sfdefault}{lmss}
\usepackage{color}
\usepackage{amssymb}

\usepackage[colorlinks,linkcolor=blue,citecolor=blue,urlcolor=blue]{hyperref}  

\usepackage[accepted]{aistats2025}
\usepackage[round]{natbib}
\usepackage[colorlinks,linkcolor=blue,citecolor=blue,urlcolor=blue]{hyperref}  

\begin{document}

%

%

\twocolumn[

\aistatstitle{Bridging the Theoretical Gap in Randomized Smoothing}

\aistatsauthor{ Blaise Delattre$^{1, 2}$ \And Paul Caillon$^{2}$  \And Quentin Barth\'elemy$^{1}$ \And Erwan Fagnou$^{2}$ \And Alexandre Allauzen$^{2, 3}$ }


\aistatsaddress{\\${}^1$Foxstream\\ Vaulx-en-Velin, France \And \\${}^2$Miles Team, LAMSADE, \\ Université Paris-Dauphine, PSL \\   Paris, France \AND ${}^3$ESPCI PSL\\ Paris, France} 
]

\begin{abstract}
  Randomized smoothing has become a leading approach for certifying adversarial robustness in machine learning models. However, a persistent gap remains between theoretical certified robustness and empirical robustness accuracy. This paper introduces a new framework that bridges this gap by leveraging Lipschitz continuity for certification and proposing a novel, less conservative method for computing confidence intervals in randomized smoothing. Our approach tightens the bounds of certified robustness, offering a more accurate reflection of model robustness in practice. Through rigorous experimentation we show that our method improves the robust accuracy, compressing the gap between empirical findings and previous theoretical results.
  We argue that investigating local Lipschitz constants and designing ad-hoc confidence intervals can further enhance the performance of randomized smoothing. These results pave the way for a deeper understanding of the relationship between Lipschitz continuity and certified robustness.
\end{abstract}

\section{INTRODUCTION}

Deep neural networks have achieved remarkable success across various domains, including computer vision and natural language processing. Despite this progress, these models remain vulnerable to adversarial attacks — small and deliberately crafted perturbations that can mislead the model towards incorrect predictions \citep{szegedy_intriguing_2014}. This vulnerability raises significant concerns, especially in safety-critical applications like autonomous driving and medical diagnostics, where robustness is crucial. As a result, researchers have developed various defenses to mitigate the risk of adversarial attacks, including empirical approaches \citep{madry2018towards} and certified defenses that provide formal guarantees of robustness.

Certified defenses aim to offer theoretical assurances that a model will remain robust against adversarial perturbations within a specified budget. Among these, methods built upon Lipschitz continuity \citep{cisse_parseval_2017, tsuzuku_lipschitz-margin_2018} have been explored to control the model's sensitivity to the magnitude of input changes. However, establishing tight and practical guarantees remains challenging, especially in high-dimensional settings.

Randomized smoothing (RS) \citep{li_second-order_2018, lecuyer_certified_2019, cohen_certified_2019} has emerged as a leading approach to certified defenses due to its probabilistic nature, which strikes a balance between certified robustness and computational efficiency. By adding Gaussian noise to inputs, RS provides a probabilistic guarantee of robustness, formalized through a certified radius within which a classifier is guaranteed to be robust. However, despite its promise, a notable gap remains between empirical and certified robustness, raising questions about RS's practical effectiveness in real-world scenarios \citep{maho2022attack}. 

Closing this gap has become a central challenge in adversarial robustness research. Several recent efforts have sought to narrow it by tailoring data-driven noise \citep{alfarra2022data}, incorporating higher-order information about the smoothed classifier \citep{mohapatra2020higher},  and exploring the interaction between the Lipschitz continuity of the underlying classifier and the smoothed classifier \citep{delattre2024lipschitz}. These approaches aim to refine the certified robust radius, a key metric for quantifying robustness, by improving our understanding of the classifier's smoothness properties. This motivates the need for a deeper examination of how smoothing techniques and Lipschitz continuity interact to affect this radius.

Moreover, in the certification process, the statistical estimation procedure plays a critical role in fully explaining this gap. Specifically, the design of exact coverage confidence intervals has a pivotal impact, given the probabilistic nature of randomized smoothing. Properly controlling the risk across multiple, non-independent random variables is essential for achieving tighter bounds on certification while guaranteeing a given risk level. This issue is linked to multiple testing problems, such as the \emph{Bonferroni} correction and family-wise error rate control, where our approach aims to avoid the "conservativeness" commonly associated with these methods. Such challenges are well-known in fields like genomics and other high-dimensional data settings \citep{shaffer1995Multiple, benjamini1995controlling, goeman2014multiple, holm1979simple, efron2010large}. While several scientific fields can be plagued with false findings due to improper multiple testing \citep{Ioannidis2005Why}, the scientific community faces the same kind of risk with end-of-pipeline certification. As a consequence, in our case, certified radii may be underestimated, causing RS methods to fail to achieve a reasonable risk level \(\alpha\).

In this paper, we propose a new certified radius for randomized smoothing that incorporates the Lipschitz continuity assumptions of the base classifier. Additionally, we introduce an efficient method for computing exact coverage confidence intervals for RS in classification tasks. Our proposed confidence intervals alleviate the burden associated with a large number of classes, making the approach more scalable and practical. Our contributions aim to bridge the gap between empirical and theoretical robustness in RS through the following key insights:

\begin{itemize}
  \item We derive a new certified robust radius for RS that accounts for the Lipschitz continuity of the base classifier.
  \item We introduce a class partitioning method (CPM) for estimating exact coverage confidence intervals in the multi-class setting, improving both certified accuracy and efficiency.
  \item We provide a detailed analysis based on local Lipschitz continuity to explain the performance gap observed in randomized smoothing defenses.
\end{itemize}

A comprehensive set of experiments demonstrates the reliability of our contributions.

\section{BACKGROUND \& RELATED WORK}


We consider a \( d \)-dimensional input \( \vx \in \mathcal{X} \subset \mathbb{R}^d \) and its corresponding label \( y \in \mathcal{Y} = \{1, \dots, c\} \), where \( \mathcal{Y} \) represents the set of \( c \) possible classes.  
The set \( \Delta^{c-1} = \left\{ \vp \in \mathbb{R}^c \mid \mathbf{1}^\top \vp = 1, \ \vp \geq 0 \right\} \) defines the \((c-1)\)-dimensional probability simplex.  
Let \( \tau: \mathbb{R}^c \mapsto \Delta^{c-1} \) represent a mapping onto this simplex, typically corresponding to functions like \( \mathrm{softmax} \) or \( \mathrm{hardmax} \).  
For a logit vector \( \vz \in \mathbb{R}^c \), the mapping onto \( \Delta^{c-1} \) is denoted by \( \tau(\vz) \).  
A specific case of this mapping is the \( \mathrm{hardmax} \), where for each component \( k \), we have \( \tau_k(\vz) = 
\mathrm{1}_{\arg\max_{i} \vz_i = k} \), 
concentrating all the probability mass on the maximum value. For any set $\mathcal{A}$, we also denote $\mathcal{\mid A \mid}$ its cardinal.

We define a function \( f: \mathcal{X} \mapsto \mathbb{R}^c \) as the network, which produces the logits before applying \( \tau \).  
The soft classifier is defined as:
$F(\vx) = \tau \circ f(\vx),$
which outputs a probability distribution over the \( c \) classes. The final decision is given by the hard classifier:
$F^H(\vx) = \arg\max_{k \in \mathcal{Y}} F_k(\vx),$
which returns the predicted label \( \hat{y} = F^H(\vx) \). While \( F^H(\vx) \) provides a classification for input \( \vx \), it does not convey the confidence level of the prediction.

To quantify the robustness of the classifier around a particular input, we introduce the certified radius. The certified radius \( R(F^H, \vx) \) measures the largest perturbation \( \epsilon \) that can be applied to \( \vx \), such that the classifier’s prediction remains unchanged. Formally, the certified radius under the \( \ell_2 \)-norm is defined as:

\begin{align*}
    R(F^H, \vx, y) := &\inf \{\epsilon  \mid \epsilon > 0,  \\ & 
    \exists
    \delta \in B_2(0, \epsilon), F^H(\vx + \mathbf{\delta}) \neq y\}
\end{align*}
where \( B_2(0, \epsilon) = \{ \delta \in \mathbb{R}^d \mid \|\delta\|_2 \leq \epsilon \} \).

We define the local Lipschitz constant of a function \( f \) with respect to the \( \ell_2 \)-norm over a set \( \mathcal{B} \). It is given by:
\[
L(f, \mathcal{B}) = \sup_{\substack{\vx, \vx' \in \mathcal{B} \\ \vx \neq \vx'}} \frac{\lVert f(\vx) - f(\vx^\prime) \rVert_2}{\lVert \vx - \vx' \rVert_2} \ .
\]
If \( L(f, \mathcal{B}) \) exists and is finite, we say that \( f \) is locally Lipschitz over \( \mathcal{B} \). When \( \mathcal{B} = \mathcal{X} \), we refer to \( L(f, \mathcal{X}) = L(f) \), the global Lipschitz constant.
%

\subsection{Lipschitz-continuous neural networks}

Lipschitz continuity has long been recognized for its crucial role in building robust classifiers. By ensuring that a function has a bounded Lipschitz constant, one can guarantee that small input perturbations will not cause large fluctuations in the output.
\cite{tsuzuku_lipschitz-margin_2018} gives 
a bound on the certified radius of the network \(f\)  using the margin at input \(\vx\) and the Lipschitz constant of $f$.
This property highlights the strength of Lipschitz continuity as a defense mechanism against adversarial attacks. Recent advancements have focused on designing classifiers with inherent Lipschitz constraints, either by incorporating regularization during training or through specialized architectures \citep{tsuzuku_lipschitz-margin_2018, trockman_orthogonalizing_2021, meunier_dynamical_2022, araujo_unified_2022, hu2023scaling}. Some approaches \citep{araujo_lipschitz_2021, singla_fantastic_2021, delattre_efficient_2023} apply soft regularization to control the Lipschitz constant of individual layers.
Computing the exact Lipschitz constant of a neural network is NP-hard \citep{virmaaux2018lipschitz}, and obtaining close bounds is often computationally expensive or arbitrarily loose, such as through the product upper bound of layers.
Probabilistic methods for constructing Lipschitz-continuous classifiers can be applied in an architecture-agnostic manner.

\subsection{Randomized smoothing}

Randomized smoothing (RS) was initially proposed by \ \cite{lecuyer_certified_2019} and further developed in key works \citep{li_second-order_2018, cohen_certified_2019, salman_provably_2020}. The core idea behind RS is to improve the robustness of a classifier by averaging its predictions over Gaussian perturbations of the input. This gives a hard smoothed classifier $\tilde{F}^H$, which is more resistant to adversarial attacks than the classifier $F^H$. 

The soft smoothed classifier $\tilde{F}$ is defined as:
\[
    \tilde{F}(\vx) = \mathbb{E}_{\mathbf{\delta} \sim \mathcal{N}(0, \sigma^2 \mI)} \left[ F(\vx + \mathbf{\delta}) \right] \ .
\]
This means that $\tilde{F}^H(\vx) = \argmax_k \tilde{F}_k(\vx)$   returns the class $k$ that has the highest probability when Gaussian noise $\mathbf{\delta}$ is added to the input $x$.
A common choice for $\tau$ is the \emph{hardmax} function, which assigns 1 to the highest score and 0 to the others, as in \cite{cohen_certified_2019, salman_provably_2020}. 

\subsubsection{Certified radius}
For ease of notation, let's denote $\vp = \tilde{F}(\vx)$, the vector of output probabilities given by the smoothed classifier, with $\vp_k$ its $k-$th component.
A key contribution of RS is that it allows us to compute a lower bound on the \emph{certified radius}.
Its estimation is based on the top two class probabilities, denoted as $\vp_{i_1}$ and  $\vp_{i_2}$ where $ i_1 = \argmax_k \vp_k(\vx)$ and $i_2 = \argmax_{k \neq i_1} \vp_k(x)$. Assuming these probabilities are sorted in decreasing order, the certified radius is given by \cite{cohen_certified_2019}:
\[
    R_{\mathrm{mult}}(\vp) = \frac{\sigma}{2} \left( \Phi^{-1}(\vp_{i_1}) - \Phi^{-1}(\vp_{i_2}) \right) \ ,
\]
where $\Phi^{-1}$ is the inverse cumulative distribution function of the standard Gaussian distribution. This radius provides a quantifiable measure of the classifier’s robustness.
A simplification is usually made which is to bound $\vp_{i_2} \leq 1- \vp_{i_1}$,
giving the bound on the radius:
    $R_{\mathrm{mono}}(\vp) = \sigma \Phi^{-1}(\vp_{i_1}) \leq \ R_{\mathrm{mult}}(\vp)$
\citep{cohen_certified_2019}.

The probabilities $\vp$ are estimated using a Monte Carlo (MC) approach. Given $n$ Gaussian samples $\mathbf{\delta}_i \sim \mathcal{N}(0, \sigma^2 I)$, the empirical estimate of $\vp$ is:
\[
    \hat{\vp} = \frac{1}{n} \sum_{i=1}^n F(\vx + \delta_i) \ .
\]
However, this Monte Carlo estimation introduces uncertainty because of the limited number of samples. The goal of certification is to compute an estimate of the radius with an exact risk level of \( (\alpha) \).
To manage this uncertainty, we follow the approach of \cite{levine_certifiably_2019}, which uses concentration inequalities to bound the true probabilities of each class for $R_{\mathrm{mult}}$ or use Pearson-Clopper as \cite{cohen_certified_2019} for $R_{\mathrm{mono}}$.

The radius \( R_{\mathrm{mono}} \) is often simpler to use, as it allows for controlling robustness with a precise confidence level \(1 - \alpha \). However, in binary classification scenarios where class \( i_1 \) is certified against a merged superclass, the probability bound \( R_{\mathrm{mono}} \) becomes ineffective when \( \vp_{i_1} \leq \frac{1}{2} \) \citep{voracek_improving_2023, voracek2024treatment}.

Alternatively, the radius \( R_{\mathrm{mult}} \), which considers the probabilities \( \vp_{i_1} \) and \( \vp_{i_2} \), can certify a non-trivial radius even when \( \vp_{i_1} < \frac{1}{2} \), see Figure~\ref{fig:comparison_r2_vs_r4}, providing more accurate estimates especially for larger values of smoothing variance \citep{delattre2024lipschitz} which corresponds to higher entropy distribution of $\vp$.

However, constructing confidence intervals for \( R_{\mathrm{mult}} \) requires extra caution, as it involves managing several interdependent quantities. Previous approaches \citep{voracek_improving_2023, voracek2024treatment} introduce errors in how they allocate the risk budget \( \alpha \). This flawed allocation breaks the Bonferroni correction by assuming dependence between variables, as discussed by a counterexample (\href{https://github.com/blaisedelattre/bridging_the_gap_rs/blob/main/counter_example.py}{link}) in the Appendix.

\subsection{Exact coverage confidence interval}

To estimate the certified radius, one must rely on approximations since the smoothed classifier cannot be computed directly. MC integration provides this approximation, but to control the risk associated with the certified radius accurately, an exact coverage confidence interval is needed.

\begin{defn}
Let \( \xi \) be a random variable. An \( (1 - \alpha) \)-exact coverage confidence interval for a parameter \(\theta\) is an interval \([\underline{\xi}, \overline{\xi}]\) such that the probability of \(\theta\) lying within the interval is greater than \(1 - \alpha\), i.e.,
\[
\mathbb{P}(\underline{\xi} \leq \theta \leq \overline{\xi}) \geq 1 - \alpha.
\]
Here, \(\alpha\) represents the risk, which is the probability that the interval does not contain the true value of \(\theta\).
\end{defn}

Commonly used methods to construct such intervals include empirical Bernstein bounds \citep{maurer_empirical_2009}, the Hoeffding inequality \citep{boucheron_concentration_2013} for continuous-valued simplex mapping, and Clopper-Pearson methods for discrete mapping (hardmax). These approaches ensure the certified radius is computed with a controlled risk, providing reliable estimates for robustness certification.

\subsection{Confidence validation for \( R_{\mathrm{mono}} \)}

In the context of certification using \( R_{\mathrm{mono}} \), the radius is determined based on the index \( i_1 \), which is obtained from independent samples, following the approach of \cite{cohen_certified_2019}. The certified radius is defined as:
\[
\mu = R_{\mathrm{mono}}(\underline{\hat{\vp}_{i_1}}) \ ,
\]
where \( \underline{\hat{\vp}_{i_1}} \) represents the lower bound of the \( (1 - \alpha) \)-exact Clopper-Pearson confidence interval for the observed proportion \( \hat{\vp}_{i_1} \). The confidence test for this certification involves the following,
\( H_0 \): The probability of the certified radius being at least \( \mu \) is at least \( 1 - \alpha \), i.e.,
\[
H_0: \mathbb{P}(R_{\mathrm{mono}}(\vp_{i_1}) \geq \mu) \geq 1 - \alpha \ .
\]
\( H_1 \): The probability of the certified radius being less than \( \mu \) is less than \( \alpha \), i.e.,
\[
H_1: \mathbb{P}(R_{\mathrm{mono}}(\vp_{i_1}) < \mu) < \alpha \ .
\]

This 
framework helps to determine whether the certified radius provides a reliable measure of robustness under the specified risk level $\alpha$.

\subsection{Confidence validation for \( R_{\mathrm{mult}} \)}

When certifying \( R_{\mathrm{mult}} \), we need to control the overall error rate across multiple comparisons, as we are dealing with several confidence intervals for different classes. The Bonferroni correction is used here to adjust the significance level, ensuring valid hypothesis testing even with dependent tests \citep{benjamini1995controlling, hochberg1987multiple, dunn1961multiple}.

\begin{thm}[\cite{hochberg1987multiple}]
\label{thm:bonferroni}
Let \( H_1, H_2, \dots, H_m \) be a family of \( m \) null hypotheses with p-values \( p_i \), and let \( \alpha \) be the desired family-wise error rate. The family-wise error rate (FWER) is defined as

\[
\text{FWER} = \mathbb{P}\left(\bigcup_{i=1}^{m} \{ \text{reject } H_i \mid H_i \text{ is true}\}\right),
\]

representing the probability of making at least one Type I error among the multiple tests. The Bonferroni correction sets the individual significance level for each test to \( \frac{\alpha}{m} \), such that

\[
(H_i) \text{ is rejected if } p_i \leq \frac{\alpha}{m} \ .
\]

This ensures control of the family-wise error rate at level \( \alpha \), even for dependent tests.
\end{thm}
To handle the multiple comparisons in our problem, we choose \( \alpha' = \frac{\alpha}{c} \), where \( c \) is the number of classes, thereby adjusting each risk level to control the overall error rate. This adjustment allows us to simultaneously validate \( (1 - \alpha) \)-level confidence intervals for all classes.
We define
\[
\mu = R_{\mathrm{mult}}(\underline{\hat{\vp}_{I_1}}, \overline{\hat{\vp}_{I_2}}) = \frac{\sigma}{2} \left( \Phi^{-1}(\underline{\hat{\vp}_{I_1}})
- \Phi^{-1}(\overline{\hat{\vp}_{I_2}}) \right),
\]
where \( I_1 \) is the index of the highest lower bound \( \underline{\hat{\vp}_i} \) among the \( (1 - \alpha') \)-level Clopper-Pearson confidence intervals, and \( I_2 \) is the index of the second-highest upper bound \( \overline{\hat{\vp}_i} \), with \( i \neq I_1 \).

By constructing these intervals at level \( 1 - \alpha' \), Theorem~\ref{thm:bonferroni} ensures that the probability of any interval failing to cover the true parameter \( \vp_i \) is at most \( \alpha \), thus validating our confidence validations. Similarly:
\begin{align*}
H_0:\ &\mathbb{P}\left( R_{\mathrm{mult}}(\boldsymbol{p}) \geq \mu \right) \geq 1 - \alpha \\
H_1:\ &\mathbb{P}\left( R_{\mathrm{mult}}(\boldsymbol{p}) < \mu \right) < \alpha \ .
\end{align*}

The choice of \( \alpha' = \frac{\alpha}{c} \) ensures correct coverage across multiple intervals. Using a different value, such as \( \alpha' = \frac{\alpha}{2} \), would lead to incorrect coverage assumptions, as shown by a counterexample in the Appendix involving a simple multinomial distribution for \( \vp \), where dependence assumptions among the confidence intervals are invalid.

\section{BRIDGING THE GAP ON THE CERTIFIED RADIUS}
Despite advances in robustness certification, a significant gap persists between theoretical bounds and empirical robustness. This section introduces novel methods to close this gap, including an efficient procedure for constructing tighter confidence intervals using class partitioning and a new certificate based on Lipschitz continuity.

\begin{figure*}[t]
    \centering
    \includegraphics[width=0.95\textwidth]{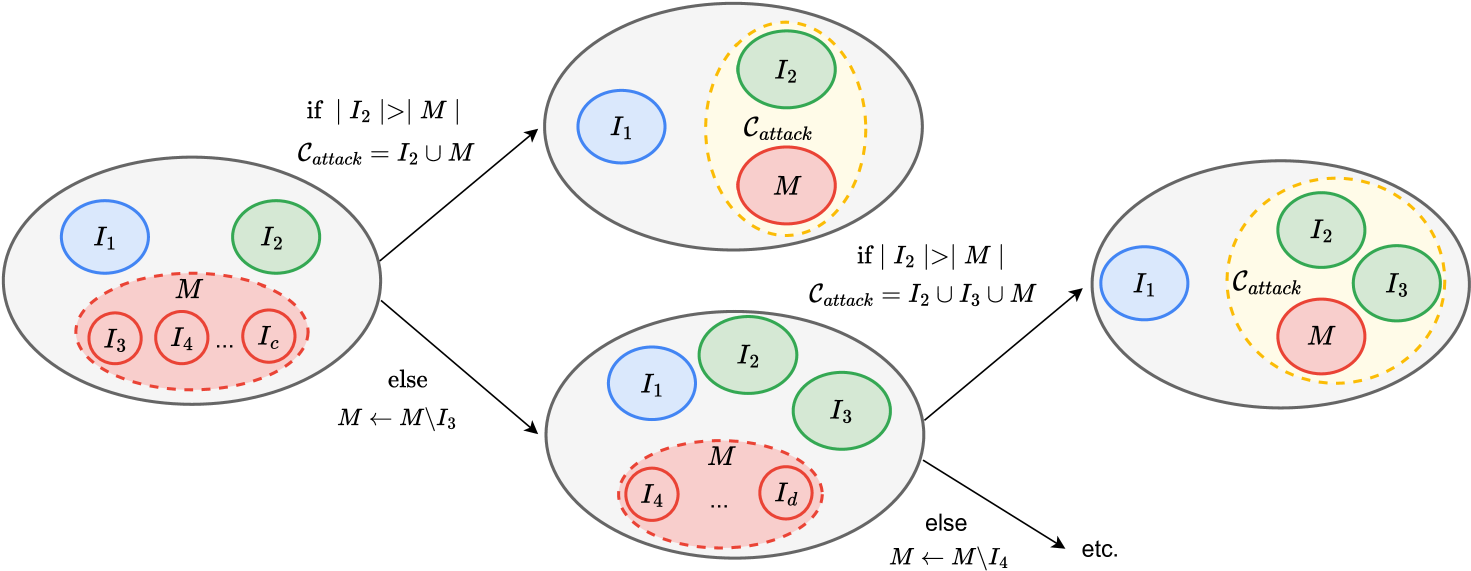}
    \caption{Illustration of the initial phase of the class partitioning method (CPM). The number of counts in class $I$ is noted $\mid I \mid$.}
    \label{fig:explanation_bucket_algo}
\end{figure*}

\subsection{New procedure for efficient confidence intervals with class partitioning.}

The robustness certification measure \( R_{\mathrm{multi}} \) performs better for smaller values of \( \hat{\vp}_{i_1} \), but it requires careful handling of confidence intervals and a conservative risk level \( \alpha' = \frac{\alpha}{c} \), where \( c \) is the number of classes. This conservative choice can be a limitation in cases with a large number of classes, such as ImageNet (\( c = 1000 \)), and a limited number of samples.

To address this, we propose a Class Partitioning Method (CPM) algorithm that combines multiple hypothesis testing using Bonferroni correction with a partitioning strategy, as detailed in Algorithm~\ref{algo:bucket_algo} in the Appendix.  The algorithm begins by drawing an initial set of \( n_0 \) samples and partitioning the classes into "buckets" based on their selection counts, as illustrated in Figure~\ref{fig:explanation_bucket_algo}. This initial grouping reduces the number of comparisons, enhancing statistical power while avoiding the computational cost of testing each class individually.

Initially, three buckets are considered: \( I_1 \), representing the most probable class; \( I_2 \), the second most probable class; and a meta-class \( M \), which includes all remaining classes. This setup enables a focus on the most relevant comparisons by treating the less significant classes together as one group.
The algorithm then iteratively refines the partitioning by evaluating whether any classes in \( M \) should be treated as separate buckets. If the total count of \( M \) exceeds that of \( I_2 \), one or more classes from \( M \) are removed and considered individually, aiming to balance the comparison. This process continues until the partitioning results in a set of buckets \( \mathcal{C}_{\mathrm{attack}} \), composed of \( I_2 \), the individually separated classes from \( M \), and the remaining meta-class \( M \).

In the estimation phase, the algorithm draws additional \( n \) samples and calculates upper confidence bounds for each bucket in \( \mathcal{C}_{\mathrm{attack}} \) and lower confidence bound for class \( I_1 \). The Bonferroni correction is applied at the bucket level, adjusting the significance level to \( \alpha' = \dfrac{\alpha}{{c^\star}} \), where \( {c^\star} = \mid \mathcal{C}_{\mathrm{attack}} \mid + 1 \) is the total number of buckets. This approach lowers the conservativeness compared to the initial setting of \( \alpha' = \frac{\alpha}{c} \), resulting in tighter confidence intervals while maintaining control over the family-wise error rate.

By transitioning from \( \alpha' = \frac{\alpha}{c} \) to \( \alpha' = \dfrac{\alpha}{{c^\star}} \), the algorithm dynamically adjusts the risk level based on the refined partitioning. This method achieves more efficient and reliable robustness guarantees for multi-class classifiers, especially in large-scale settings, e.g., ImageNet (\( c = 1000 \)). 
Although these improved confidence intervals help tighten the bounds on the certified radius, they do not fully explain the robustness gap observed with existing methods.

\subsection{New certificate with Lipschitz continuity}

We enhance the framework by incorporating Lipschitz continuity assumptions on the network to address this gap, as explored in \cite{chen2024diffusion, delattre2024lipschitz}. This assumption is reasonable, given that classifiers trained with noise \citep{salman_denoised_2020} or built on diffusion models \citep{carlini_certified_2023} naturally develop local Lipschitz properties, as demonstrated in Table~\ref{tab:lip_comparison}.

The certified radius bounds obtained from randomized smoothing (RS) improve significantly as the probability \( \hat{\vp}_{I_1} \) approaches 1, causing the radius to grow nonlinearly toward infinity. In contrast, the bound derived in \cite{tsuzuku_lipschitz-margin_2018} exhibits a linear dependence on \( \hat{\vp}_{I_1} \), leading to a more gradual increase in the certified radius. Consequently, RS provides substantially stronger robustness guarantees when \( \hat{\vp}_{I_1} \) is high. As shown in \cite{cohen_certified_2019}, the RS bounds consistently outperform the linear Lipschitz-based bound from \cite{tsuzuku_lipschitz-margin_2018} regarding certified robustness.

Most existing works continue to use the standard RS radii 
\( R_{\mathrm{mono}} \) and \( R_{\mathrm{multi}} \), 
which implicitly assume a globally smooth classifier but do not explicitly account for its local Lipschitz properties.
The gap between theoretical RS bounds and practical performance can be partially explained by the overlooked role of local Lipschitz continuity.
While Lipschitz-based bounds such as those in~\citet{tsuzuku_lipschitz-margin_2018} provide general robustness guarantees,
they do not account for the locally varying smoothness of neural networks.
Our approach refines these bounds by integrating local Lipschitz continuity into the RS framework, 
providing a new perspective on how smoothness impacts certified radii.

%
The following theorem provides a key result for deriving certified radii under the assumption of Lipschitz continuity on the function \( F \) for the randomized smoothing framework.
\begin{thm}[Randomized Smoothing with Local Lipschitz Continuity]
\label{thm:lip_local_rs}
Let \( \sigma > 0 \) and let \( F : \mathbb{R}^d \to [0, 1] \) be a Lipschitz continuous function. Define the smoothed classifier:
$\tilde{F}(\vx) = \mathbb{E}_{\boldsymbol{\delta} \sim \mathcal{N}(0, \sigma^2 \mathbf{I})} \left[ F(\vx + \boldsymbol{\delta}) \right] $ .
Then, for any \( \vx \in \mathbb{R}^d \), the function \( \Phi^{-1} \circ \tilde{F} \) is locally Lipschitz continuous within the ball \( B(\vx, \rho) = \{ \vx^\prime \in \mathbb{R}^d : \lVert \vx^\prime - \vx \rVert_2 \leq \rho \} \), with Lipschitz constant:

\begin{align*}
    &L\left( \Phi^{-1} \circ \tilde{F}, B(\vx, \rho) \right) 
     = \\
    &L(F) \sup_{\vx^\prime \in B(\vx, \rho)} \left\{ 
     \frac{  \Phi_\sigma\left(s_0(\vx^\prime) + \dfrac{1}{L(F)}\right) - \Phi_\sigma(s_0(\vx^\prime))}
    {
    \phi(
        \Phi^{-1}(\tilde{F}(\vx^\prime))
        )
    } 
    \right\} \
\end{align*}

where \( \Phi_\sigma \) is the cumulative distribution function of the normal distribution with standard deviation \( \sigma \), $\phi$ is the density of the standard normal distribution, and \( s_0(\vx^\prime) \) is determined for each \( \vx^\prime \in B(\vx, \epsilon) \) by solving:
\[
\tilde{F}(\vx^\prime) = 1 - L(F) \int_{s_0(\vx^\prime)}^{s_0(\vx^\prime) + \frac{1}{L(F)}} \Phi_\sigma(s) \, ds \ .
\]
\end{thm}
$s_0$ can be computed using numerical root finding Brent’s method \citep{brent1973algorithms} with $0$ machine precision.
This result contrasts with the standard RS framework, where robustness certificates are derived without explicitly accounting for the classifier’s smoothness.

A key challenge in applying this result in practice is the computation of local Lipschitz constants, 
which are difficult to estimate precisely with certification.
To estimate the local Lipschitz constant we rely on~\citet{yang2020closer} using a gradient ascent approach (similar to PGD for crafting adversarial attack). \\
Figure~\ref{fig:comparison_r2_vs_r4} illustrates that 
incorporating Lipschitz continuity of $F$ leads to a more refined characterization of the certified radius for predictions $\tilde{F}(\vx)$, our approach suggests that better Lipschitz estimates—particularly local ones—could further close the gap between theoretical guarantees and empirical observations.

\subsection{Certificate radius with local Lipschitz}
The previous section's result provides the foundation for introducing tighter certified radii by incorporating the Lipschitz continuity of the function \( F \). When \( F \) is assumed to be Lipschitz continuous, the Lipschitz properties of the smoothed classifier can be used to derive stronger bounds on the certified radius.
We define two new certified radii that leverage these Lipschitz properties:

\begin{align*}
    R_{\mathrm{multiLip}}(\vp) = &\frac{1}{2} \left( \frac{\Phi^{-1}(\vp_{i_1})}{L\left( \Phi^{-1} \circ \tilde{F}_{i_1}, B(\vx, \rho) \right)} \right. \\
    &\left. - \frac{\Phi^{-1}(\vp_{i_2})}{L\left( \Phi^{-1} \circ \tilde{F}_{i_2}, B(\vx, \rho) \right)} \right) \ ,
\end{align*}

where \( \vp_1 \) and \( \vp_2 \) represent the top two probabilities associated with the prediction.
Simplifying with bound $\vp_2 \leq 1 - \vp_1$ we get 

\begin{align*}
    R_{\mathrm{monoLip}}(\vp) = \frac{\Phi^{-1}(\vp_1)}{L\left( \Phi^{-1} \circ \tilde{F}_{i_1}, B(\vx, \rho) \right)} \ .
\end{align*}

These radii provide tighter robustness guarantees by accounting for the Lipschitz continuity of the base classifier, leading to stronger certified bounds compared to standard randomized smoothing methods that do not incorporate this assumption, they include locality with neighborhood radius $\rho$.
\begin{figure}
    \centering
    \includegraphics[width=1.0 \linewidth]{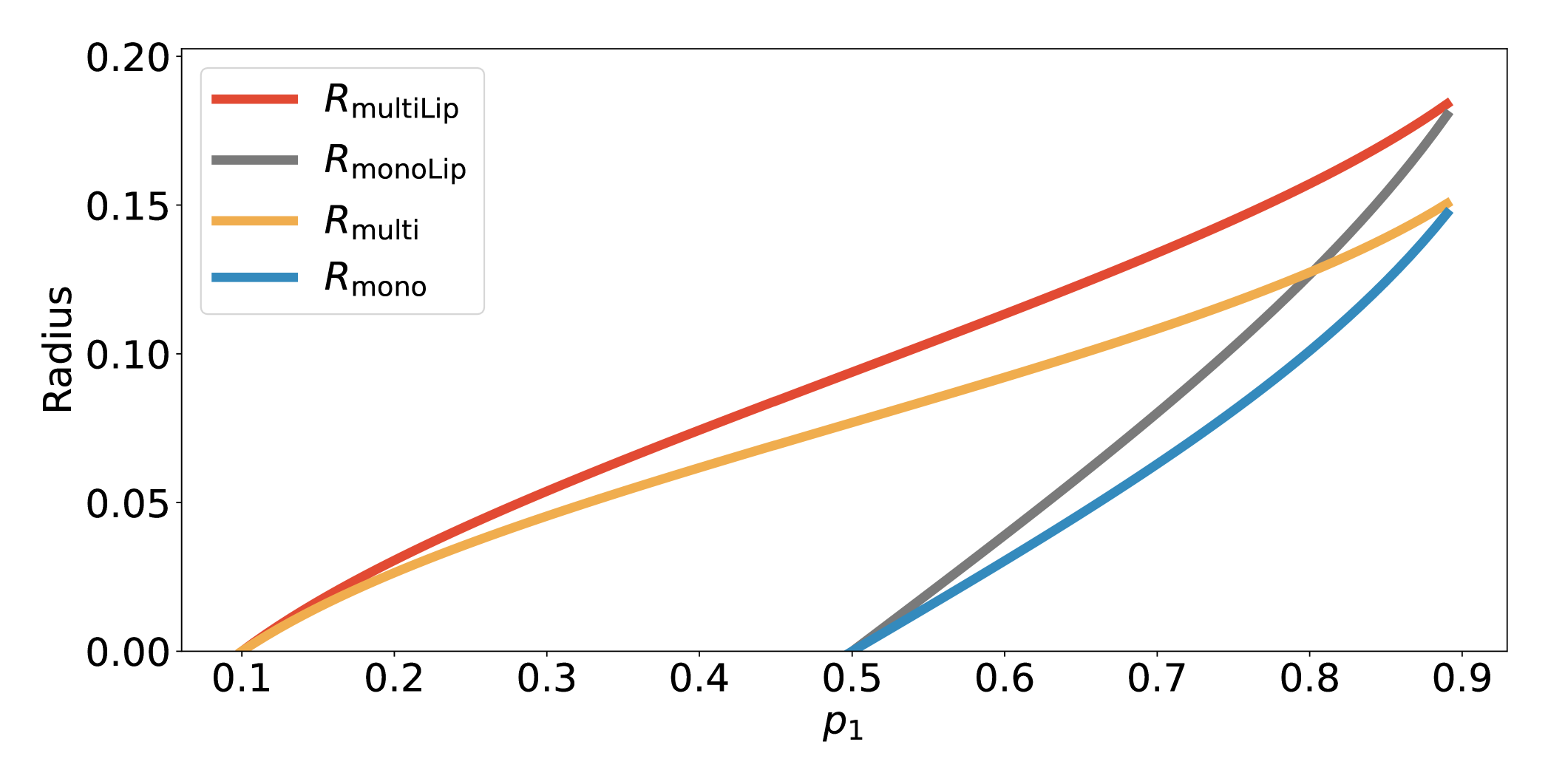}
    \caption{Comparison of radii as a function of $\vp_1$, for $L(F)=4$, $\sigma=0.12$ and $\vp_2 = 0.1$. }
    \label{fig:comparison_r2_vs_r4}
\end{figure}

Figure~\ref{fig:comparison_r2_vs_r4} shows that with a $4$-Lipschitz soft classifier \( F \), the radius \( R_{\mathrm{mono}} \) underestimates the certified radius \( R_{\mathrm{monoLip}} \), demonstrating that the Lipschitz continuity assumption tightens the bound on the true radius. Moreover, the Lipschitz-adjusted radii (\( R_{\mathrm{multiLip}} \) and \( R_{\mathrm{monoLip}} \)) consistently outperform the original radii. 
However, it is important to note that the radii presented in Figure~\ref{fig:comparison_r2_vs_r4} are approximations, as they rely on an estimation of the local Lipschitz constant rather than its exact computation. Despite this, the overall trend remains valid: incorporating Lipschitz continuity into the RS framework leads to tighter certified bounds and helps explain the gap between theoretical and empirical robustness in RS.

For large \( \vp_1 \) values or large \( \sigma \), 
the behavior of the Lipschitz-adjusted radii converges to that of standard RS, 
indicating that the influence of local Lipschitz properties becomes less significant when the classifier is already highly confident 
or when the smoothing parameter dominates.
However, in the low-confidence regime, Lipschitz-aware bounds provide substantially tighter guarantees, 
revealing that smoothness plays a crucial role in regions of uncertainty.
This insight helps explain why RS models often exhibit stronger empirical robustness than predicted by standard theoretical bounds.

%
The choice of configuration for plotting the certified radii—specifically, the Lipschitz constant \( L(F) = 4 \), \( \sigma = 0.12 \), and \( \vp_2 = 0.1 \)—is based on values obtained from a LiResNet model trained with data augmentation using Gaussian noise of standard deviation $0.5$. This setup is intended to illustrate the potential gains achievable with the Lipschitz-adjusted radii. 

\label{exp:exp_certif_imagenet_comparison_ic}
\begin{figure*}[t]
    \centering
    \includegraphics[width=1\linewidth]{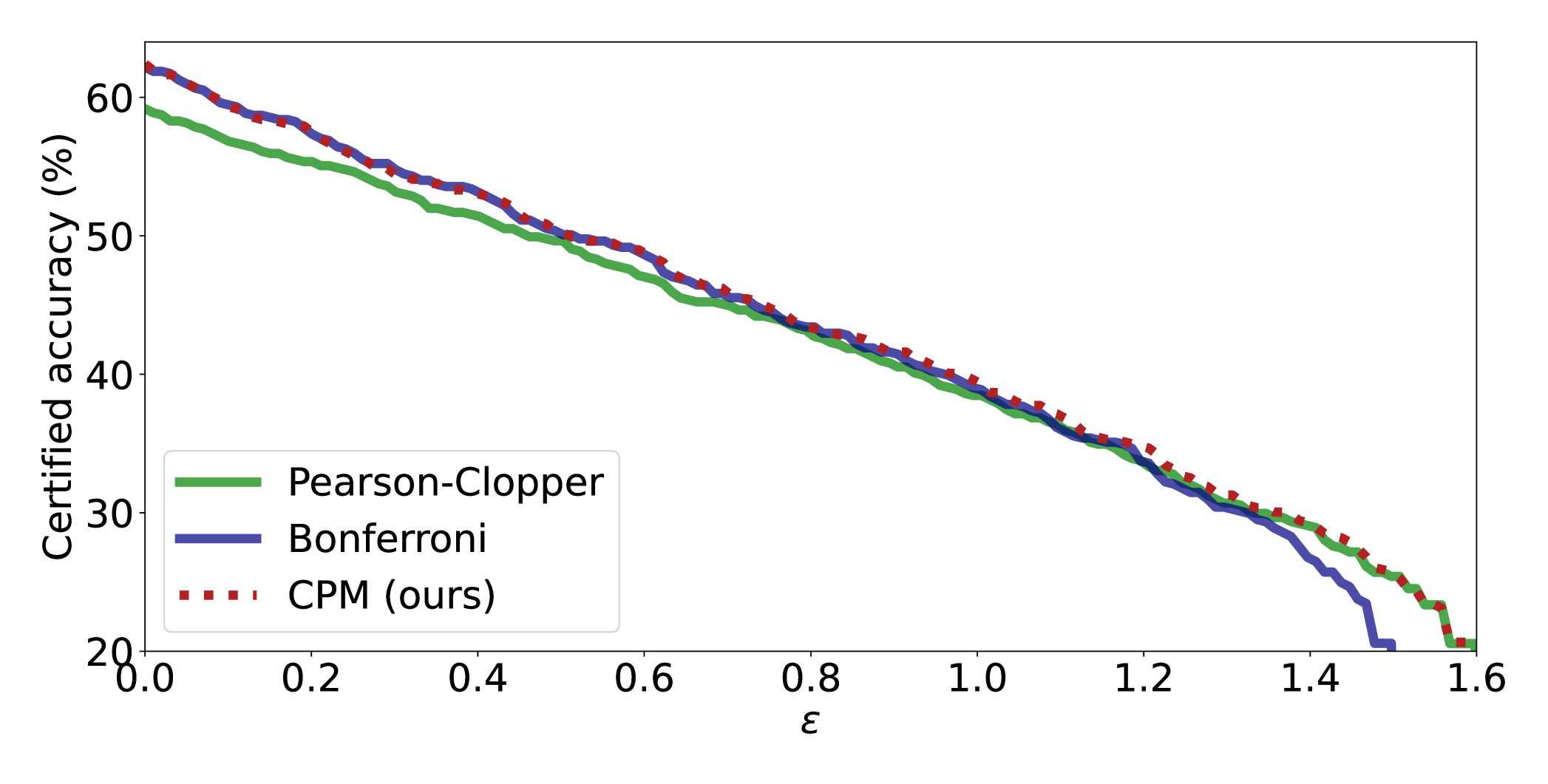}
    \caption{Comparison of various confidence interval methods for certified accuracy estimation with smoothing standard deviation $\sigma=0.5$ on the ImageNet dataset, using ResNet-50 trained with noise injection ($\sigma=0.5$). The plot contrasts our CPM method with Bonferroni and Pearson-Clopper intervals across different perturbation levels $\epsilon$.}
    \label{fig:ic^Heap_bf_pc-label}
\end{figure*}
\section{EXPERIMENTS}
For all experiments, we use a subset of 500 images for ImageNet and 1,000 images for CIFAR-10. The risk level is set to \(\alpha = 0.001\), with \(n = 10^4\) samples used for estimation and \(n_0 = 100\) used for initialization. All experiments are conducted using a single A100 GPU, you can find the code on \href{https://github.com/blaisedelattre/bridging_the_gap_rs}{\textbf{github}}.
\subsection{Certificate with new confidence interval}
We evaluate our methods on the task of image classification using the ImageNet dataset, employing a ResNet-50 architecture trained with noise injection (\(\sigma=0.5\)). The randomized smoothing procedure adds Gaussian noise to the inputs, providing probabilistic robustness guarantees. We report certified accuracy as a function of the level of \(\ell_2\) perturbation \(\epsilon\), which quantifies the size of adversarial perturbations against which the model is certified to be robust.

For certification, we use the certified radius \(R_{\text{mult}}\) for CPM and the standard Bonferroni correction over \(c\) classes, while \(R_{\text{mono}}\) is used for Pearson-Clopper intervals. Our CPM method typically produces an effective number of class $c^\star$ between $3$ and $5$ providing a less conservative approach to certification. These gains are registered with no additional computational cost compared to the inference cost of the network. As shown in Figure~\ref{fig:ic^Heap_bf_pc-label}, our CPM method achieves a balance between the strong performance of Bonferroni intervals at high perturbation levels and the robustness of Pearson-Clopper intervals at lower perturbations. In the Appendix, we provide more graphs for different $\sigma$ and the CIFAR-10 dataset (10 classes).
This demonstrates its effectiveness in providing reliable certified accuracy across a wide range of \(\epsilon\) values.

\subsection{Randomized smoothing with Lipschitz certificate}
\begin{figure}
    \centering
    \includegraphics[width=1\linewidth]{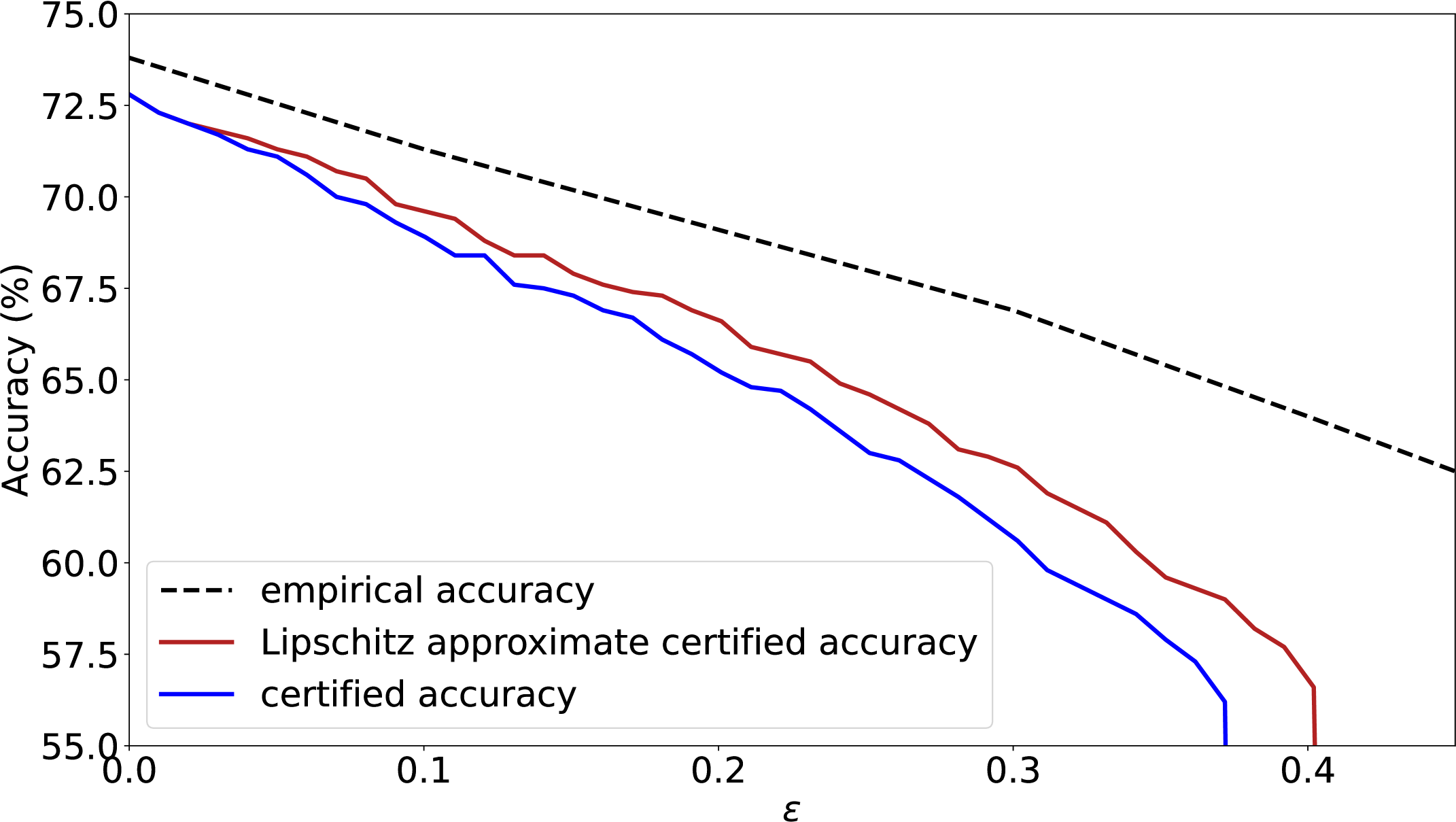}
    \caption{LiResNet trained with noise injection ($\sigma=0.5$), on CIFAR-10 dataset. Certified accuracy $R_{\mathrm{multi}}$ vs Lipschitz estimated certified accuracy $R_{\mathrm{multiLip}}$ vs empirical robust accuracy using projected gradient ascent. Randomized smoothing noise is taken as ($\sigma=0.12$).}
    \label{fig:acc_certif_vs_empirical_vs_lip_liresnet}
\end{figure}
To compute \( R_{\mathrm{multiLip}} \), we approximate the local Lipschitz constant by setting \( \rho=0 \) in $L\left( \Phi^{-1} \circ \tilde{F}, B(\vx, \rho) \right) $, as exact this local Lipschitz estimates is difficult to compute exactly.  The reported certified accuracy in Figure~\ref{fig:acc_certif_vs_empirical_vs_lip_liresnet} should be interpreted as an approximation to explain the gap rather than an exact guarantee.
We used LiResNet \citep{hu2023scaling}, a state-of-the-art Lipschitz network for certified robustness, trained on the CIFAR-10 dataset. We compared the standard certified robustness obtained with the traditional RS procedure using \( R_{\mathrm{multi}} \), and our new Lipschitz-adjusted radius, \( R_{\mathrm{multiLip}} \), both employing the standard Bonferroni correction and empirical Bernstein confidence bound \citep{maurer_empirical_2009}.
For empirical assessment, we conducted empirical evaluations using a projected gradient descent (PGD) \( \ell_2 \) attack \citep{madry2018towards} with 40 iterations and a step size of 0.2. This setup highlights the potential improvements in certified robustness with the Lipschitz-adjusted radius.

In Figure~\ref{fig:acc_certif_vs_empirical_vs_lip_liresnet}, we observe that the certified radius obtained with the Lipschitz continuity assumption, \( R_{\mathrm{multiLip}} \), is closer to the empirical robust accuracy (accuracy on the test set under  PGD \( \ell_2 \) attack). 
Incorporating the Lipschitz assumption refines the certified bound, better approximating robustness. 
The gap between standard RS-certified accuracy and empirical robustness stems partly from ignoring Lipschitz's continuity. 
While \( R_{\mathrm{multiLip}} \) narrows this gap, it remains an approximation due to the challenge of computing exact local Lipschitz constants, 
underscoring the role of smoothness in bridging theory and practice.
%
%
Ideally, the empirical robust accuracy evaluation should involve attacking the smoothed classifier directly as in \cite{salman_provably_2020} or consider stronger attacks than PGD. As a result, the empirical robustness reported is likely to be slightly overestimated, since attacking the smoothed classifier is typically more challenging and would yield lower robustness values. 

\subsection{Lipschitz bound of neural networks}

Neural networks used in RS are theoretically Lipschitz continuous, but estimating a meaningful Lipschitz constant is challenging. 
A common approach to obtain a cheap upper bound is using the product of each layer's Lipschitz constants. Let \( l_i \) denote  the \( i \)-th layer, then the product upper bound ($\mathrm{PUB}$) for the Lipschitz constant of the network \( F \) is given by:
$\text{L}(F) \leq \mathrm{PUB}(F) = \prod_{i=1}^{m} L(l_i)$,
where \( m \) is the total number of layers in the network. While this bound is easy to compute, it can be extremely loose for deeper networks, as the actual local Lipschitz behavior may be much smaller than this product.
Lipschitz networks control the overall Lipschitz constant by design, using a product of upper bounds across layers and making it close to true Lipschitz. However, this global upper bound can be arbitrarily loose compared to the local Lipschitz constant around a given input point \(x\), particularly for standard architectures like ResNet-110. As a result, the certified radius derived using this global bound does not leverage the benefits of Lipschitz continuity, as it does not reflect the local behavior of the network near \(x\).

\begin{table}[h!]
\centering
\caption{Comparison of local Lipschitz constant estimation (Local Lip.) and product upper bound (PUB) for different models and noise levels (\( \sigma \)) used during training.}
\begin{tabular}{@{}lcc@{}}
\toprule
\textbf{Model}                      & \textbf{Local Lip.} & \textbf{PUB} \\ \midrule
LiResNet (\( \sigma = 0.00 \))      & 22                  & 37                        \\
LiResNet (\( \sigma = 0.12 \))      & 11                  & 13                        \\
LiResNet (\( \sigma = 0.25 \))      & 7.9                 & 9                         \\
LiResNet (\( \sigma = 0.50 \))      & 4.8                 & 6                         \\ \midrule
ResNet-110 (\( \sigma = 0.00 \))    & 235                 & \( 2.34 \times 10^{10} \) \\
ResNet-110 (\( \sigma = 0.12 \))    & 27                  & \( 1.03 \times 10^{12} \) \\
ResNet-110 (\( \sigma = 0.25 \))    & 25                  & \( 3.19 \times 10^{12} \) \\
ResNet-110 (\( \sigma = 0.50 \))    & 19                  & \( 9.71 \times 10^{9} \)  \\
ResNet-110 (\( \sigma = 1.0 \))     & 3.8                 & \( 1.32 \times 10^{11} \) \\ 
\bottomrule
\end{tabular}
\label{tab:lip_comparison}
\end{table}

Table~\ref{tab:lip_comparison} illustrates the significant discrepancy between local Lipschitz estimates and the $\mathrm{PUB}$, especially for ResNet-110, where the global bound becomes impractically loose. This gap underscores the limitations of relying on PUB for certification, motivating the need for research focused on deriving certified bounds for local Lipschitz constants.
To explain why the empirical radius is often much larger than the theoretical one, we estimate the local Lipschitz constant \(L(F, \mathcal{B}(\vx, \epsilon)\) with a radius of \( \epsilon =0.15 \) with gradient ascent. This local approach better captures the network's behavior near \(x\), providing a more accurate estimation of the certified radius. However, since local estimates do not guarantee a strict upper bound, they cannot be directly used for certification.
Training with noise injection helps to reduce the local Lipschitz constant, making the network smoother and more robust. For example, increasing noise levels \( \sigma \) significantly decreases the Lipschitz constant, as shown in Table~\ref{tab:lip_comparison}. This reduction not only smooths the decision boundaries but also reduces the classifier's variance, leading to tighter MC confidence intervals as shown in \cite{delattre2024lipschitz}.

Establishing rigorous methods for certified local Lipschitz bounds~\citep{huang2021local} is crucial for enhancing the certified robustness of neural networks and aligning theoretical guarantees with empirical observations. 
Also, diffusion models and denoising methods have been explored to reduce the Lipschitz constant, thereby improving certified robustness \citep{salman_denoised_2020, carlini_certified_2023}.
\section{CONCLUSION}
We addressed the gap between empirical and certified robustness in randomized smoothing by introducing a new certified radius based on Lipschitz continuity and an efficient class partitioning method for exact coverage confidence intervals. Our findings highlight that local Lipschitz continuity and local certified estimation are crucial for producing tighter and more reliable robustness certificates. 

\section{ACKNOWLEDGMENTS}
This work was performed using HPC resources from GENCI- IDRIS (Grant 2023-AD011014214R1) and funded by the French National Research Agency (ANR SPEED-20-CE23-0025). This work received funding from the French Government via the program France 2030 ANR-23-PEIA-0008, SHARP.

\bibliography{ref}
\bibliographystyle{plainnat}

\section*{Checklist}



 \begin{enumerate}

 \item For all models and algorithms presented, check if you include:
 \begin{enumerate}
   \item A clear description of the mathematical setting, assumptions, algorithm, and/or model. Yes
   \item An analysis of the properties and complexity (time, space, sample size) of any algorithm. Yes
   \item (Optional) Anonymized source code, with specification of all dependencies, including external libraries. Yes, will be available upon acceptance.
 \end{enumerate}

 \item For any theoretical claim, check if you include:
 \begin{enumerate}
   \item Statements of the full set of assumptions of all theoretical results. Yes
   \item Complete proofs of all theoretical results. Yes
   \item Clear explanations of any assumptions. Yes
 \end{enumerate}

 \item For all figures and tables that present empirical results, check if you include:
 \begin{enumerate}
   \item The code, data, and instructions needed to reproduce the main experimental results (either in the supplemental material or as a URL). Yes, will be available upon acceptance.
   \item All the training details (e.g., data splits, hyperparameters, how they were chosen). Yes
         \item A clear definition of the specific measure or statistics and error bars (e.g., with respect to the random seed after running experiments multiple times). Not Applicable
         \item A description of the computing infrastructure used. (e.g., type of GPUs, internal cluster, or cloud provider). Yes
 \end{enumerate}

 \item If you are using existing assets (e.g., code, data, models) or curating/releasing new assets, check if you include:
 \begin{enumerate}
   \item Citations of the creator If your work uses existing assets. Yes
   \item The license information of the assets, if applicable. Not Applicable
   \item New assets either in the supplemental material or as a URL, if applicable. Not Applicable
   \item Information about consent from data providers/curators. Not Applicable
   \item Discussion of sensible content if applicable, e.g., personally identifiable information or offensive content. Not Applicable
 \end{enumerate}

 \item If you used crowdsourcing or conducted research with human subjects, check if you include:
 \begin{enumerate}
   \item The full text of instructions given to participants and screenshots. Not Applicable
   \item Descriptions of potential participant risks, with links to Institutional Review Board (IRB) approvals if applicable. Not Applicable
   \item The estimated hourly wage paid to participants and the total amount spent on participant compensation. Not Applicable
 \end{enumerate}

 \end{enumerate}
\newpage
\appendix

%
%




%

%

\onecolumn

\section{Appendix}

\subsection{Class partitioning method}
We describe the to perform the Class Partitioning Method.
\begin{algorithm}
    \caption{Class partitioning method (CPM) for optimized Bonferroni correction}
    \label{algo:bucket_algo}
    \begin{algorithmic}[1]
        \Require Initial sample size \( n_0 \), estimation sample size \( n \), risk level \( \alpha \)
        \State \textbf{Initialize:} Class set \( \mathcal{C} = \{1, 2, \dotsc, c\} \)

        \Statex

        \State \textbf{Initial Sampling Phase:}

        \State Draw \( n_0 \) samples and compute the selection counts \( C_{\mathrm{select}}[i] \) for each class \( i \in \mathcal{C} \)

        \State \( I_1 \gets \arg\max_{i \in \mathcal{C}} C_{\mathrm{select}}[i] \) \Comment{Index of the class with the highest count}

        \State \( I_2 \gets \arg\max_{i \in \mathcal{C} \setminus \{I_1\}} C_{\mathrm{select}}[i] \) \Comment{Index of the class with the second-highest count}

        \State Initialize meta-class \( M \gets \mathcal{C} \setminus \{I_1, I_2\} \)

        \State Define \( \mathcal{C}_{\mathrm{attack}} \gets \{I_2, M\} \)

        \While{ \( \sum_{i \in M} C_{\mathrm{select}}[i] > C_{\mathrm{select}}[I_2] \) }
        \State \( k \gets \arg\max_{i \in M} C_{\mathrm{select}}[i] \)
        \State Move class \( k \) from \( M \) to \( \mathcal{C}_{\mathrm{attack}} \)
        \EndWhile

        \State Let \( \mathcal{P} \) be the partitioning to form \( \mathcal{C}_{\mathrm{attack}} \)

        \Statex

        \State \textbf{Estimation Phase:}

        \State Draw \( n \) samples and compute the selection counts \( C_{\mathrm{estim}}[i] \) for each class \( i \in \mathcal{C} \)

        \State Apply partitioning \( \mathcal{P} \) to \( C_{\mathrm{estim}} \) to obtain bucket counts

        \State Compute \( \hat{\vp}_{I_1} = \dfrac{C_{\mathrm{estim}}[I_1]}{n} \)

        \For{ each bucket \( B \) in \( \mathcal{C}_{\mathrm{attack}} \) }
        \If{ \( B \) is a single class \( k \) }
        \State Compute \( \hat{\vp}_k = \dfrac{C_{\mathrm{estim}}[k]}{n} \)
        \Else \Comment{ \( B \) is meta-class \( M \) }
        \State Compute \( \hat{\vp}_M = \dfrac{\sum_{i \in M} C_{\mathrm{estim}}[i]}{n} \)
        \EndIf
        \EndFor

        \State Compute the total number of buckets \( c^\star = |\mathcal{C}_{\mathrm{attack}}| + 1 \)

        \State Compute the adjusted significance level \( \alpha' = \dfrac{\alpha}{c^\star} \)

        \State Compute the lower confidence bound \( \underline{\hat{\vp}}_{I_1} \) with risk \( \alpha' \)

        \For{ each bucket \( B \) in \( \mathcal{C}_{\mathrm{attack}} \) }
        \State Compute the upper confidence bound \( \overline{\hat{\vp}}_B \) with risk \( \alpha' \)
        \EndFor

        \State Let \( \overline{\hat{\vp}}_{\max} = \max_{B \in \mathcal{C}_{\mathrm{attack}}} \overline{\hat{\vp}}_B \)

        \Statex

        \State \textbf{Output:} Robustness radius \( R(\underline{\hat{\vp}}_{I_1}, \overline{\hat{\vp}}_{\max}) \)
    \end{algorithmic}
\end{algorithm}

\textbf{In the initial sampling phase}, the algorithm begins by drawing an initial sample of size \( n_0 \) and computing the selection counts \( C_{\mathrm{select}}[i] \) for each class \( i \) in the set \( \mathcal{C} = \{1, 2, \dotsc, c\} \). It identifies \( I_1 \) as the class with the highest selection count and \( I_2 \) as the class with the second-highest count. The remaining classes are grouped into a meta-class \( M = \mathcal{C} \setminus \{I_1, I_2\} \), and the set of classes to attack, \( \mathcal{C}_{\mathrm{attack}} \), is initialized as \( \{I_2, M\} \). The algorithm then refines this partitioning by iteratively moving the class \( k \) with the highest count from \( M \) to \( \mathcal{C}_{\mathrm{attack}} \) whenever the total count of \( M \) exceeds that of \( I_2 \). This process continues until the counts are balanced, resulting in a partitioning \( \mathcal{P} \) that effectively reduces the number of comparisons by grouping less significant classes.

\textbf{In the estimation phase}, the algorithm draws an additional sample of size \( n \) and computes the selection counts \( C_{\mathrm{estim}}[i] \) for each class \( i \). It then applies the partitioning \( \mathcal{P} \) from the initial phase to group these counts into buckets. The empirical probabilities are calculated for the most probable class \( \hat{\vp}_{I_1} = \frac{C_{\mathrm{estim}}[I_1]}{n} \) and for each bucket \( B \) in \( \mathcal{C}_{\mathrm{attack}} \).
The total number of buckets is determined as \( {c^\star} = |\mathcal{C}_{\mathrm{attack}}| + 1 \), and the significance level is adjusted using the Bonferroni correction to \( \alpha' = \frac{\alpha}{{c^\star}} \).
The algorithm computes the lower confidence bound \( \underline{\hat{\vp}}_{I_1} \) for \( \hat{\vp}_{I_1} \) and the upper confidence bounds \( \overline{\hat{\vp}}_B \) for each bucket \( \hat{\vp}_B \) with the adjusted risk \( \alpha' \).
It identifies the maximum upper bound \( \overline{\hat{\vp}}_{\max} = \max_{B \in \mathcal{C}_{\mathrm{attack}}} \overline{\hat{\vp}}_B \).
Finally, it outputs the robustness radius \( \mathrm{R}(\underline{\hat{p}}_{I_1}, \overline{\hat{\vp}}_{\max}) \), which provides a measure of the classifier's robustness against adversarial attacks.

\subsection{Additional experiments on CPM}

We conduct additional experiments on certified prediction margin (CPM) using ResNet-110 for CIFAR-10 and ResNet-50 for ImageNet, both trained with noise injection. The training procedure follows the approach of Cohen et al. (2019), and all settings remain consistent with those described in Section 4.1 of the paper.

For CIFAR-10, we report certified accuracy for ResNet-110, trained with different noise standard deviations (0.12, 0.25, 0.5, 1), comparing confidence intervals from Pearson-Clopper, Bonferroni, and CPM. The results are illustrated in Figure~\ref{fig:cifar10_sigma_0.12}, Figure~\ref{fig:cifar10_sigma_0.25}, Figure~\ref{fig:cifar10_sigma_0.5}, and Figure~\ref{fig:cifar10_sigma_1.0}.

For ImageNet, we evaluate certified accuracy for ResNet-50, trained with different noise standard deviations(0.25, 1), using the same confidence intervals. The results are shown in Figure~\ref{fig:imagenet_sigma_0.25},and Figure~\ref{fig:imagenet_sigma_1.0}.
The figure for $\sigma=0.5$ is already presented in the main body.

The results indicate that for high variance in probability outputs, often associated with larger noise and higher entropy, Bonferroni performs better than Pearson-Clopper, while CPM achieves results comparable to Bonferroni. Conversely, when the probability outputs exhibit lower entropy, typically with smaller noise levels, Pearson-Clopper performs better, and CPM closely mimics Pearson-Clopper while also providing superior results in intermediate scenarios. 

We observed slightly better gains on ImageNet, likely due to the larger number of classes, where CPM has a more significant impact, as Bonferroni tends to be overly conservative in such settings.

\begin{figure*}
    \centering
    \includegraphics[width=0.7\linewidth]{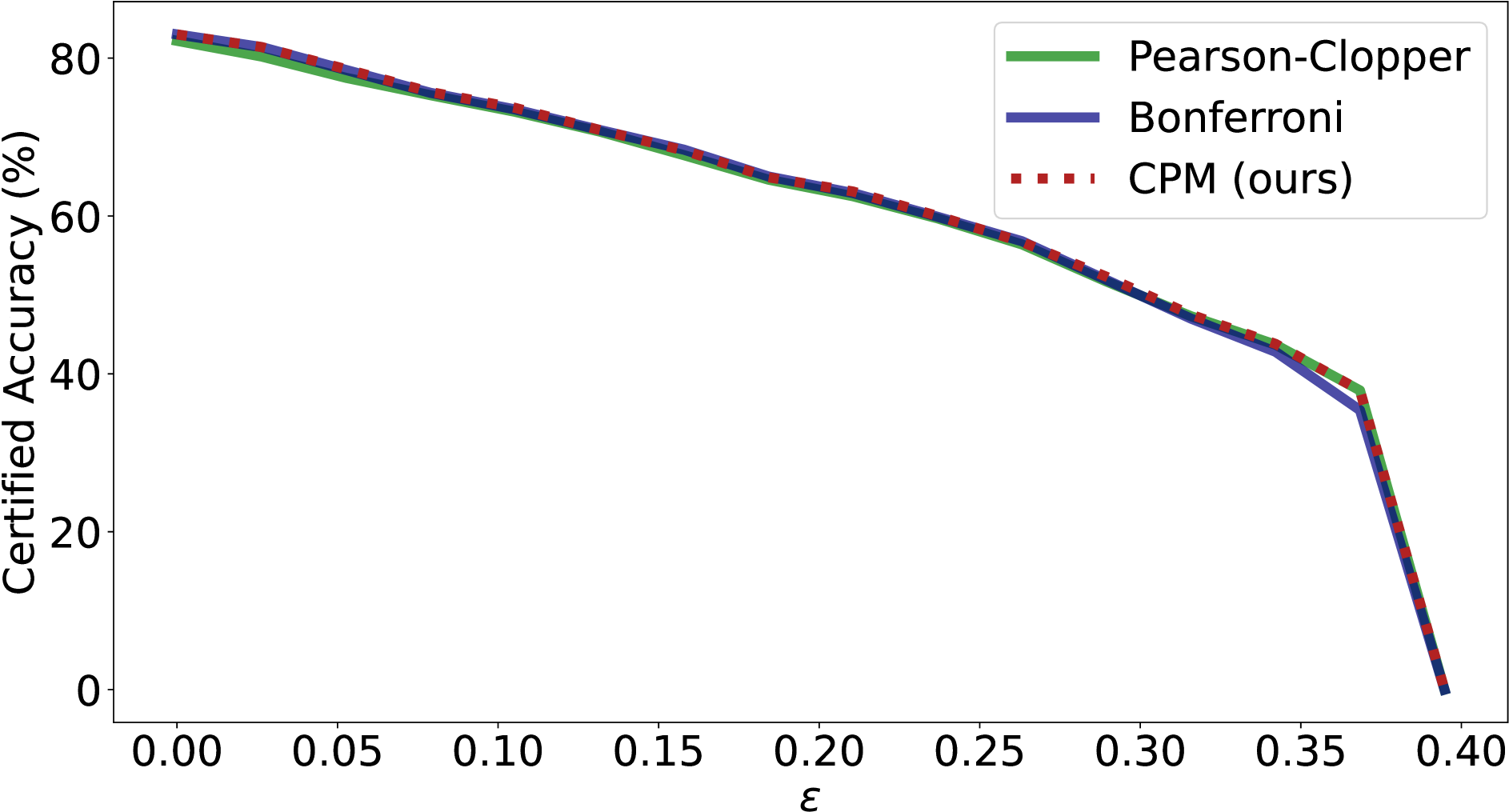}
    \caption{Comparison of various confidence interval methods for certified accuracy estimation with smoothing standard deviation $\sigma=0.12$ on the CIFAR-10 dataset.}
    \label{fig:cifar10_sigma_0.12}
\end{figure*}

\begin{figure*}
    \centering
    \includegraphics[width=0.7\linewidth]{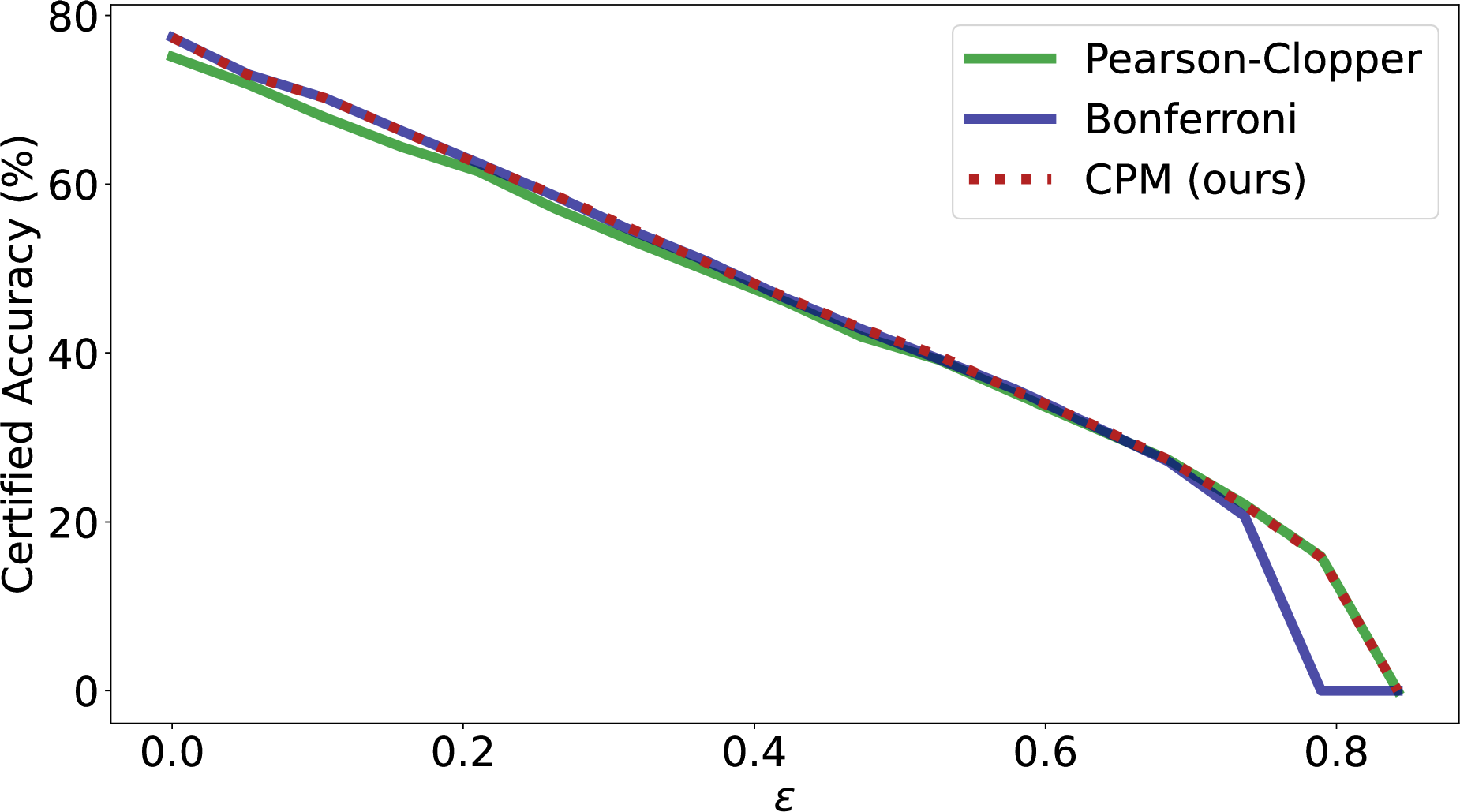}
    \caption{Comparison of various confidence interval methods for certified accuracy estimation with smoothing standard deviation $\sigma=0.25$ on the CIFAR-10 dataset.}
    \label{fig:cifar10_sigma_0.25}
\end{figure*}

\begin{figure*}
    \centering
    \includegraphics[width=0.7\linewidth]{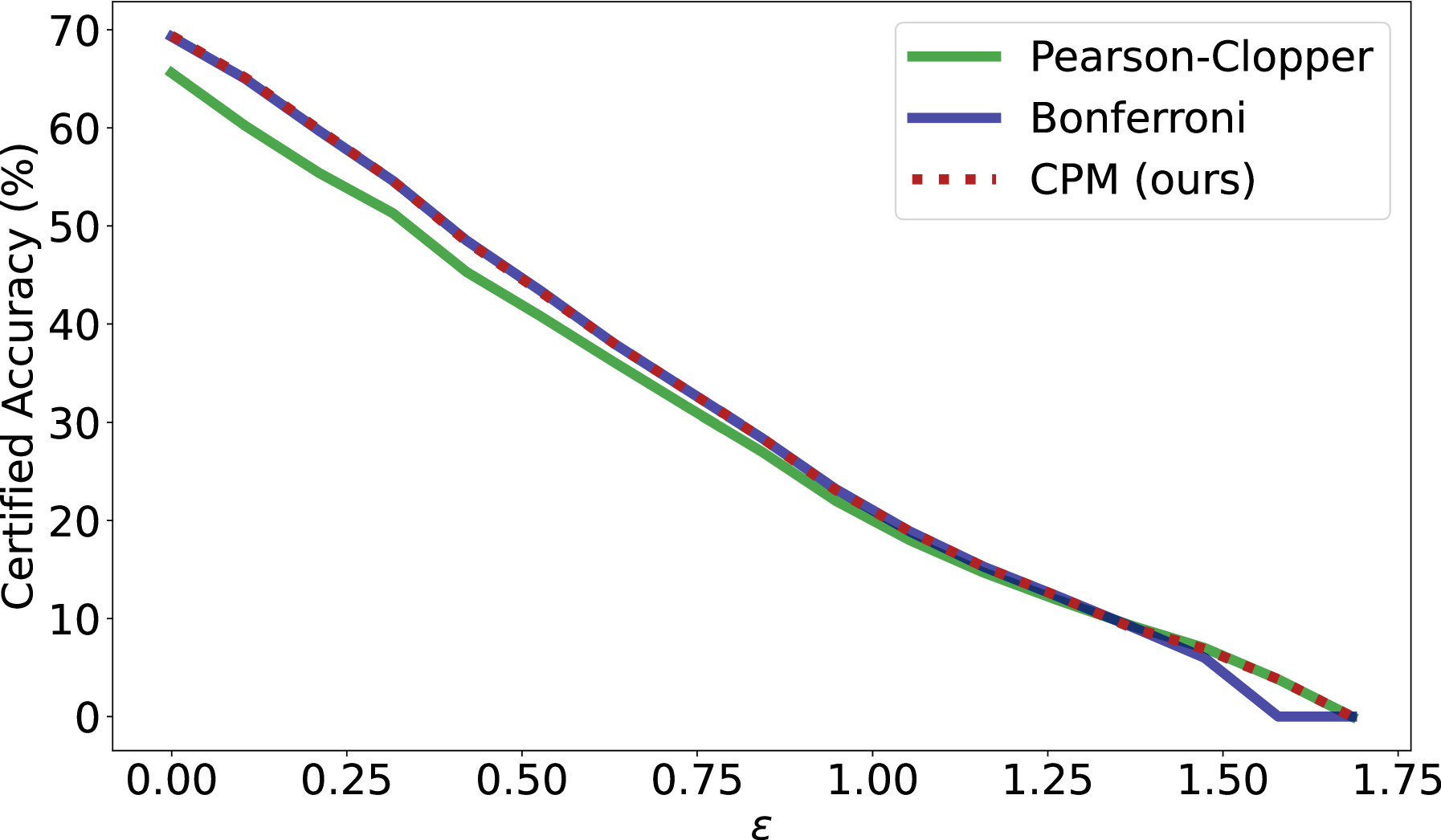}
    \caption{Comparison of various confidence interval methods for certified accuracy estimation with smoothing standard deviation $\sigma=0.5$ on the CIFAR-10 dataset.}
    \label{fig:cifar10_sigma_0.5}
\end{figure*}

\begin{figure*}
    \centering
    \includegraphics[width=0.7\linewidth]{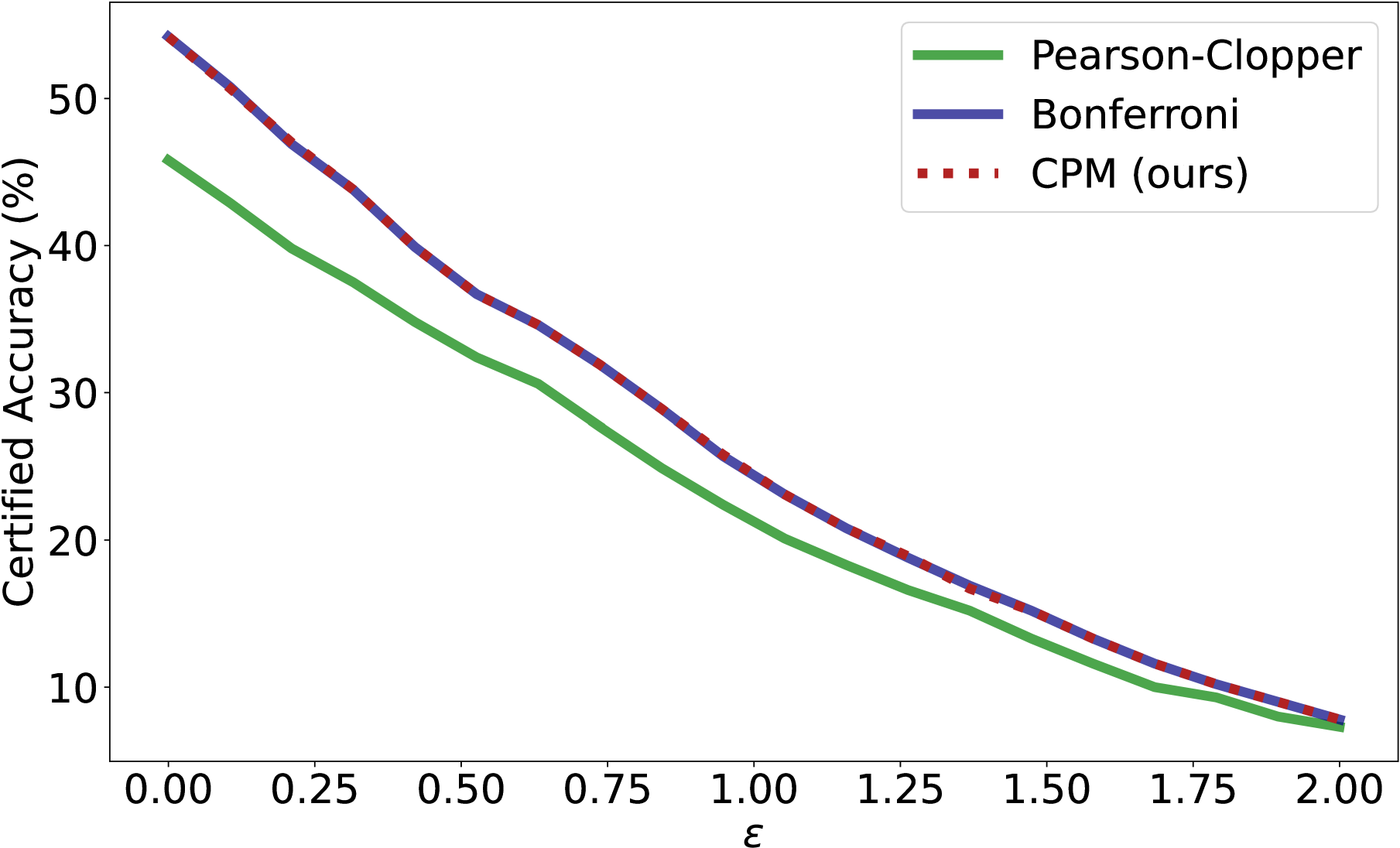}
    \caption{Comparison of various confidence interval methods for certified accuracy estimation with smoothing standard deviation $\sigma=1.0$ on the CIFAR-10 dataset.}
    \label{fig:cifar10_sigma_1.0}
\end{figure*}

\begin{figure*}
    \centering
    \includegraphics[width=0.7\linewidth]{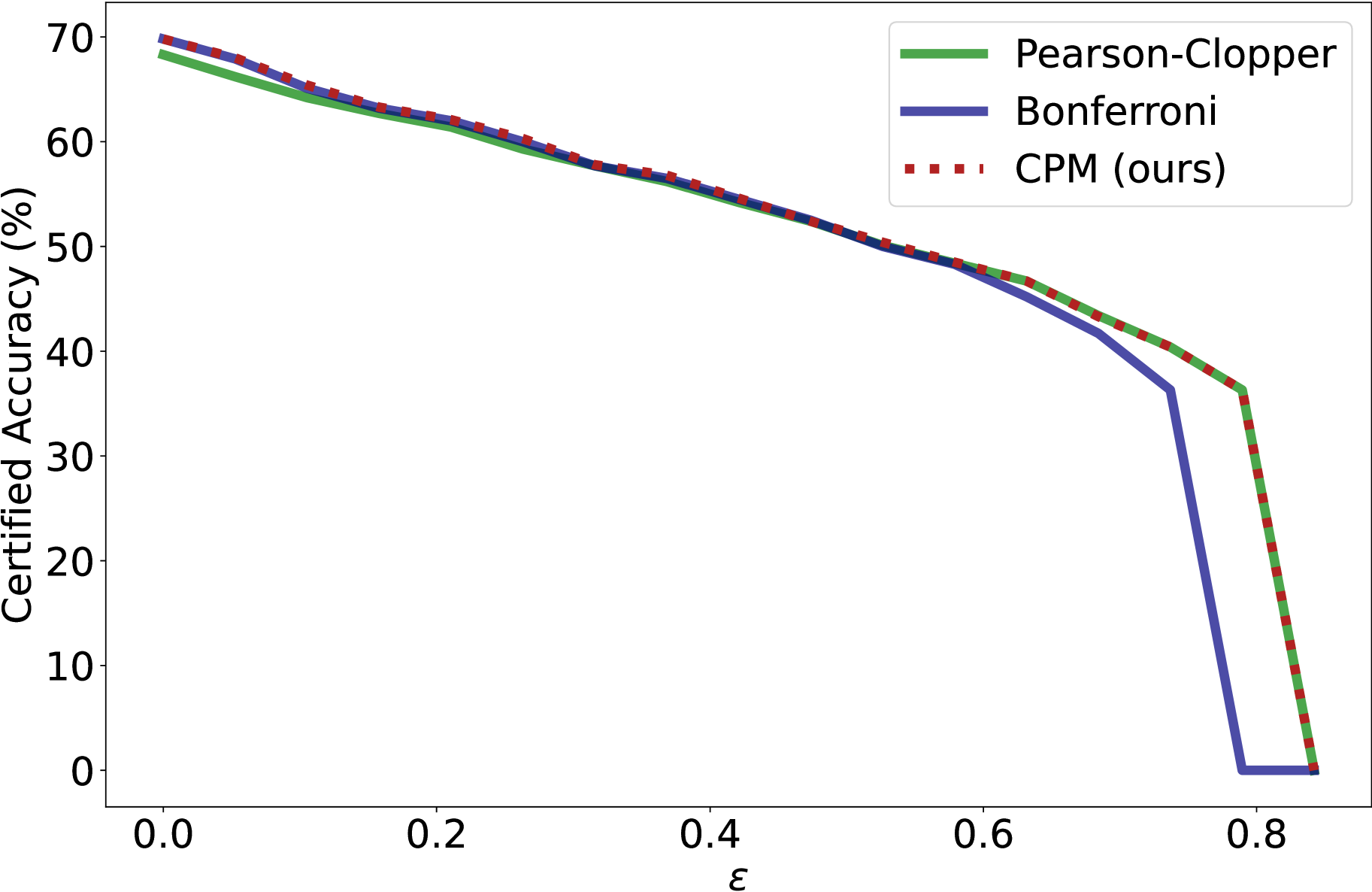}
    \caption{Comparison of various confidence interval methods for certified accuracy estimation with smoothing standard deviation $\sigma=0.25$ on the ImageNet dataset.}
    \label{fig:imagenet_sigma_0.25}
\end{figure*}


\begin{figure*}
    \centering
    \includegraphics[width=0.7\linewidth]{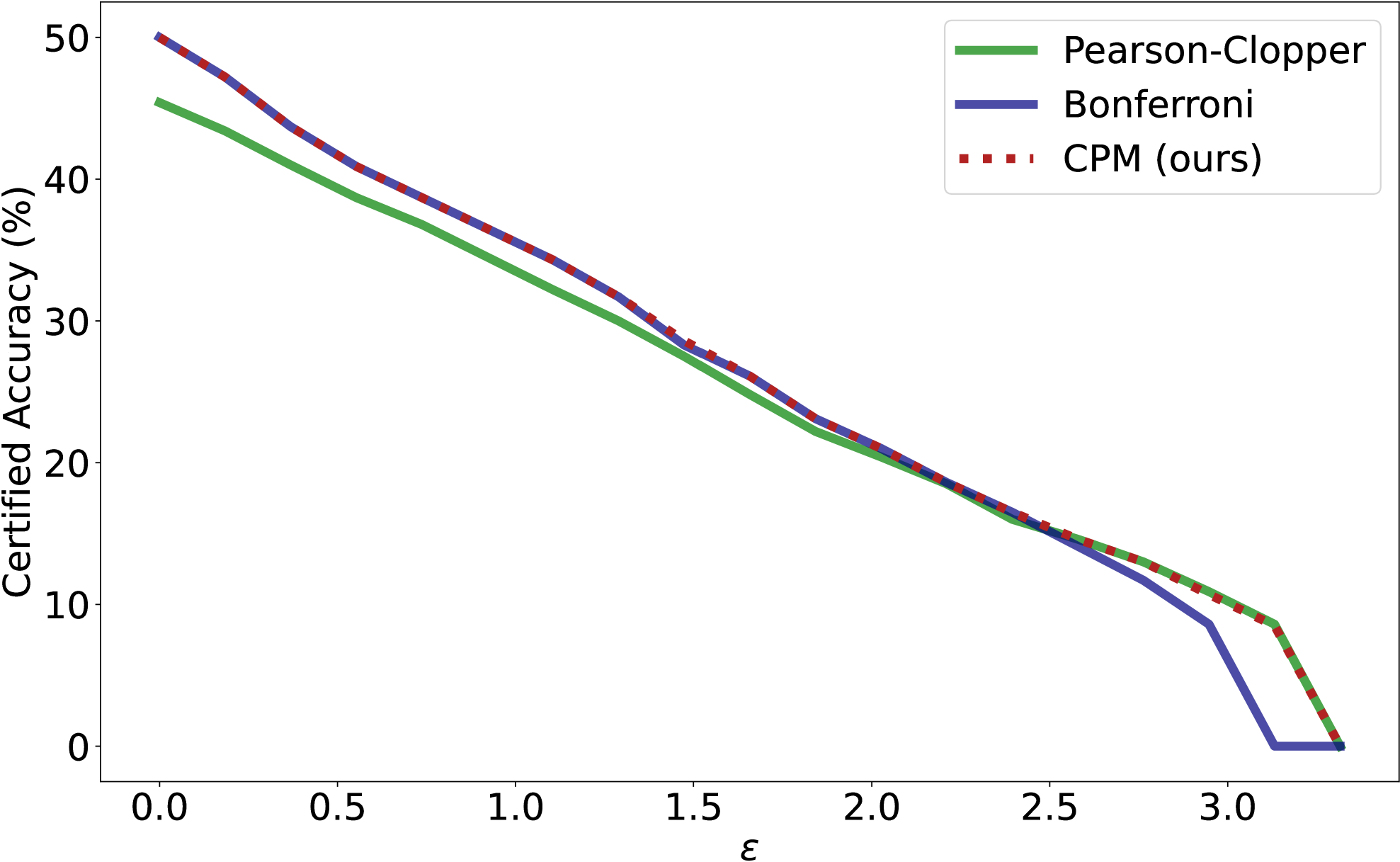}
    \caption{Comparison of various confidence interval methods for certified accuracy estimation with smoothing standard deviation $\sigma=1.0$ on the ImageNet dataset.}
    \label{fig:imagenet_sigma_1.0}
\end{figure*}

\newpage

\subsection{Counter example on why $\alpha / 2$ is erroneous}




This example illustrates how not properly applying the Bonferroni correction leads to incorrect confidence intervals and statistical inferences. The corresponding code can be found in the following repository: \href{https://github.com/blaisedelattre/bridging_the_gap_rs/blob/main/counter_example.py}{counter example}.

The code defines a function to compute the probability of an event under a multinomial distribution and demonstrates the effect of applying and not applying the Bonferroni correction.

\textbf{Proper Bonferroni Correction}: Adjusts the significance level \( \alpha \) by dividing it by the number of comparisons (e.g., \( \alpha' = \alpha / k \)) to control the family-wise error rate.

\textbf{Improper Application}: Using the unadjusted \( \alpha \) fails to account for multiple comparisons, leading to narrower confidence intervals and an increased risk of Type I errors.

When the Bonferroni correction is properly applied, the calculated probability of the event meets or exceeds the theoretical minimum probability (\( 1 - \alpha \)), maintaining the desired confidence level.
Without the correction, the probability of the event may fall below the theoretical minimum, indicating that the confidence intervals are too narrow and do not provide the intended level of confidence.
Properly adjusting for multiple comparisons is crucial for valid statistical inference, especially when making simultaneous inferences about multiple parameters.
Ensuring accurate confidence intervals helps maintain the integrity of research findings and prevents the reporting of false positives.

\subsection{Proof bound on radius is conservative on true radius}

Since \( \Phi^{-1} \) is a strictly increasing function, we have:

\[
\Phi^{-1}( \hat{\vp}_{I_1} ) \geq \Phi^{-1}( \underline{\hat{\vp}_{I_1}} )
\quad \text{and} \quad
\Phi^{-1}( \hat{\vp}_{I_2} ) \leq \Phi^{-1}( \overline{\hat{\vp}_{I_2}} ).
\]

It follows that:

\[
\Phi^{-1}( \hat{\vp}_{I_1} ) - \Phi^{-1}( \hat{\vp}_{I_2} ) \geq \Phi^{-1}( \underline{\hat{\vp}_{I_1}} ) - \Phi^{-1}( \overline{\hat{\vp}_{I_2}} ).
\]

Let

\[
\mu = \Phi^{-1}( \hat{\vp}_{I_1} ) - \Phi^{-1}( \hat{\vp}_{I_2} ).
\]

Then, we obtain:

\[
R_{\text{mult}}(\vp) = \frac{\sigma}{2} \left( \Phi^{-1}( \vp_{i_1} ) - \Phi^{-1}( \vp_{i_2} ) \right) \geq \mu.
\]

This inequality holds with probability at least \( 1 - \alpha \).

\subsection{Proof of Theorem~\ref{thm:lip_local_rs}}

We recall Stein's lemma.

\begin{lemma} \label{lemma:stein}
(Stein's lemma \cite{stein_annals_1981}) \\
Let $\sigma > 0$, let $F : \mathbb{R}^d \mapsto \mathbb{R}$ be measurable, and let 
$\tilde{h}(x) = \mathbb{E}_{\delta\sim {\cal N}(0, \sigma^2 I)}[F(x + \delta)]$. Then $\tilde{F}$ is differentiable, and moreover,
$$ 
\nabla \tilde{F}(x) = \frac{1}{\sigma^2} \mathbb{E}_{\delta\sim {\cal N}(0, \sigma^2 I)}[ \delta F(x + \delta)] \ .
$$  
\end{lemma}

We recall the Pontryagin's maximum principle (PMP).

\begin{thm}[Pontryagin's Maximum Principle] \citep{pontryagin1962mathematical}
Consider the optimal control problem:

Minimize (or maximize) the cost functional
\[
J(u) = \int_{t_0}^{t_f} L(t, x(t), u(t)) \, dt + \Phi(x(t_f)),
\]
subject to the state (dynamical) equations
\[
\dot{x}(t) = f(t, x(t), u(t)), \quad x(t_0) = x_0,
\]
and the control constraints
\[
u(t) \in U \subset \mathbb{R}^m,
\]
where \( x(t) \in \mathbb{R}^n \), \( u(t) \in \mathbb{R}^m \), and \( t \in [t_0, t_f] \).

Assume that \( L \), \( f \), and \( \Phi \) are continuously differentiable with respect to their arguments and that an optimal control \( u^*(t) \) and corresponding state trajectory \( x^*(t) \) exist.

Then there exists an absolutely continuous costate (adjoint) function \( p(t) \in \mathbb{R}^n \) and a constant \( p_0 \leq 0 \) (for minimization) or \( p_0 \geq 0 \) (for maximization) such that the following conditions are satisfied almost everywhere on \( [t_0, t_f] \):

1. State equation:
   \[
   \dot{x}^*(t) = f\big(t, x^*(t), u^*(t)\big).
   \]

2. Costate equation:
   \[
   \dot{p}(t) = -\frac{\partial H}{\partial x}\big(t, x^*(t), u^*(t), p(t)\big),
   \]
   with terminal condition
   \[
   p(t_f) = \frac{\partial \Phi}{\partial x}\big(x^*(t_f)\big).
   \]
   The Hamiltonian \( H \) is defined by
   \[
   H(t, x, u, p) = L(t, x, u) + p^\top f(t, x, u).
   \]

Maximum principle (optimality condition):
   \[
   H\big(t, x^*(t), u^*(t), p(t)\big) = \max_{v \in U} H\big(t, x^*(t), v, p(t)\big).
   \]

Transversality conditions:
   If the final time \( t_f \) is free, then
   \[
   H\big(t_f, x^*(t_f), u^*(t_f), p(t_f)\big) = 0.
   \]

Non-triviality condition:
   \[
   (p_0, p(t)) \neq 0 \quad \text{for all} \quad t \in [t_0, t_f].
   \]
\end{thm}

We are going to prove the following lemma

\begin{lemma}
\label{lemma:optimal_control_problem}
Given the optimal control problem:

Maximize the functional
\[
I(y) = \int_{-\infty}^{+\infty} s y(s) \phi(s) \, ds
\]
subject to the state equation
\[
\frac{d y(s)}{d s} = u(s),
\]
the control constraint
\[
|u(s)| \leq L,
\]
the state constraints
\[
0 \leq y(s) \leq 1,
\]
and the integral constraint
\[
\int_{-\infty}^{+\infty} y(s) \phi(s) \, ds = p,
\]
where \(\phi(s)\) is the standard normal probability density function (PDF),  \(\Phi(s)\) is the cumulative distribution function (CDF), and $p \in [0, 1]$.

The optimal control \(u^*(s)\) and the optimal state \(y^*(s)\) are given by:

When \( s \leq s_0 \):
   \[
   y^*(s) = 0, \quad u^*(s) = 0.
   \]

When \( s_0 < s < s_1 \):
   \[
   y^*(s) = L (s - s_0), \quad u^*(s) = L.
   \]

When \( s \geq s_1 \):
   \[
   y^*(s) = 1, \quad u^*(s) = 0.
   \]

Here, \( s_0 \) is determined by the equation

%
\[
   p  - \left(1 - L \int_{s_0}^{s_0 + \frac{1}{L}} \Phi(s) \, ds \right) = 0.
\]
The optimal value of the objective functional \( I(y^*) \) is:
   \[
   I(y^*) = L [ \Phi(s_0 + \frac{1}{L}) - \Phi(s_0) ].
   \]
\end{lemma}

\begin{proof}

To solve the optimal control problem, we apply the PMP with state constraints incorporated via Lagrange multipliers.

To formulate the Hamiltonian with state constraints,
we introduce Lagrange multipliers \(\mu_0(s) \geq 0\) and \(\mu_1(s) \geq 0\) for the state constraints \( y(s) \geq 0 \) and \( y(s) \leq 1 \), respectively. 

The augmented Hamiltonian is:
\[
H(s, y, p_y, u, \mu_0, \mu_1) = [s y - \lambda y] \phi(s) + p_y u + \mu_0 y + \mu_1 ( - y + 1 ),
\]
where \(\lambda\) is the Lagrange multiplier associated with the integral constraint.

Necessary Conditions from the PMP are:

State equation:
\[
\frac{d y(s)}{d s} = u(s).
\]

Costate equation:
\[
\frac{d p_y(s)}{d s} = -\frac{\partial H}{\partial y} = -[s - \lambda] \phi(s) - \mu_0(s) + \mu_1(s).
\]

Hamiltonian maximization condition:
\[
u^*(s) = \operatorname*{arg\,max}_{|u(s)| \leq L} [p_y(s) u(s)].
\]

Complementary slackness conditions:
\begin{align*}
\mu_0(s) \geq 0, \quad y(s) \geq 0, \quad \mu_0(s) y(s) = 0, \\
\mu_1(s) \geq 0, \quad y(s) \leq 1, \quad \mu_1(s) [ y(s) - 1 ] = 0.
\end{align*}

We consider three regions based on the value of \( y(s) \):

\paragraph{Region where \( y(s) = 0 \) (lower state constraint active):}

Complementary slackness: \( \mu_0(s) \geq 0 \), \( \mu_1(s) = 0 \).

Costate equation:
\[
\frac{d p_y(s)}{d s} = -[s - \lambda] \phi(s) - \mu_0(s).
\]
Admissible controls: \( u(s) \geq 0 \) (to prevent \( y(s) \) from decreasing below zero).

Optimal control: Since \( p_y(s) \leq 0 \) (due to the effect of \( \mu_0(s) \)), the Hamiltonian is maximized by \( u^*(s) = 0 \).

Resulting State: \( y(s) = 0 \).

\paragraph{Region where \( 0 < y(s) < 1 \) (state constraints inactive):}

Complementary slackness: \( \mu_0(s) = 0 \), \( \mu_1(s) = 0 \).

Costate equation:
\[
\frac{d p_y(s)}{d s} = -[s - \lambda] \phi(s).
\]

Solution of the costate equation:
\[
p_y(s) = -\int_s^{+\infty} [\sigma - \lambda] \phi(\sigma) \, d\sigma = \lambda Q(s) - \phi(s),
\]
where \( Q(s) = \int_s^{+\infty} \phi(\sigma) \, d\sigma \).

The optimal control is:
\[
u^*(s) = 
\begin{cases}
L, & \text{if } p_y(s) > 0, \\
-L, & \text{if } p_y(s) < 0.
\end{cases}
\]

We determining the switching point \( s_0 \):
\[
p_y(s_0) = 0 \implies \lambda Q(s_0) = \phi(s_0).
\]

Since \( p_y(s) > 0 \) for \( s > s_0 \), the optimal control in this region is \( u^*(s) = L \).

The resulting state is:
\[
y(s) = y(s_0) + L (s - s_0).
\]

\paragraph{Region where \( y(s) = 1 \) (upper state constraint active):}

Complementary slackness: \( \mu_0(s) = 0 \), \( \mu_1(s) \geq 0 \).

Costate equation:
\[
\frac{d p_y(s)}{d s} = -[s - \lambda] \phi(s) + \mu_1(s).
\]

Admissible controls: \( u(s) \leq 0 \) (to prevent \( y(s) \) from increasing above one).

Optimal control: since \( p_y(s) \geq 0 \) (due to the effect of \( \mu_1(s) \)), the Hamiltonian is maximized by \( u^*(s) = 0 \).

Resulting state: \( y(s) = 1 \).

Determining \( s_0 \) and \( s_1 \):

At \( s = s_0 \), \( y(s_0) = 0 \) (transition from \( y(s) = 0 \) to \( y(s) > 0 \)).
At \( y(s_1) = 1 \), \( s_1 = s_0 + \frac{1}{L} \).

Integral constraint:
\[
p = \int_{s_0}^{s_1} y(s) \phi(s) \, ds + \int_{s_1}^{+\infty} \phi(s) \, ds.
\]

Solving for \( \lambda \), \( s_0 \), and \( s_1 \) requires numerical methods due to the integrals involved.

By incorporating the state constraints into the Hamiltonian via Lagrange multipliers and applying the necessary conditions from the PMP, we derive the optimal control and state trajectories that satisfy all constraints:

Optimal Control \( u^*(s) \):
\[
u^*(s) = 
\begin{cases}
0, & s \leq s_0, \\
L, & s_0 < s < s_1, \\
0, & s \geq s_1.
\end{cases}
\]

Optimal state \( y^*(s) \):
\[
y^*(s) = 
\begin{cases}
0, & s \leq s_0, \\
L (s - s_0), & s_0 < s < s_1, \\
1, & s \geq s_1.
\end{cases}
\]

This solution ensures that the Hamiltonian is maximized over the admissible controls while satisfying the state and control constraints.

We will derive the expressions for \( p \) and \( I(y^*) \) in terms of \( s_0 \), \( L \), and the standard normal CDF \( \Phi(s) \).

\paragraph{Expression for \( p \):}

The integral constraint is:
\[
p = \int_{-\infty}^{+\infty} y^*(s) \phi(s) \, ds.
\]

Using the structure of \( y^*(s) \):
\[
y^*(s) = 
\begin{cases}
0, & s \leq s_0, \\
L (s - s_0), & s_0 < s < s_1, \\
1, & s \geq s_1.
\end{cases}
\]

We can split the integral into regions:
\[
p = \int_{s_0}^{s_1} y^*(s) \phi(s) \, ds + \int_{s_1}^{+\infty} y^*(s) \phi(s) \, ds.
\]

Compute each part:

first integral (\( p_1 \)):
  \[
  p_1 = \int_{s_0}^{s_1} L (s - s_0) \phi(s) \, ds = L \int_{s_0}^{s_1} (s - s_0) \phi(s) \, ds.
  \]

  We perform integration by parts:
  
  Let \( u(s) = (s - s_0) \), \( dv(s) = \phi(s) \, ds \).

  Then \( du(s) = ds \), \( v(s) = \Phi(s) \).

  Integration by parts gives:
  \[
  \int_{s_0}^{s_1} (s - s_0) \phi(s) \, ds = \left[ (s - s_0) \Phi(s) \right]_{s_0}^{s_1} - \int_{s_0}^{s_1} \Phi(s) \, ds.
  \]

  Evaluating the boundaries:
  \[
  (s_1 - s_0) \Phi(s_1) - (s_0 - s_0) \Phi(s_0) = \frac{1}{L} \Phi(s_1).
  \]

  Therefore,
  \[
  \int_{s_0}^{s_1} (s - s_0) \phi(s) \, ds = \frac{1}{L} \Phi(s_1) - \int_{s_0}^{s_1} \Phi(s) \, ds.
  \]

  Multiply both sides by \( L \):
  \[
  p_1 = \Phi(s_1) - L \int_{s_0}^{s_1} \Phi(s) \, ds.
  \]

Second integral (\( p_2 \)):
  \[
  p_2 = \int_{s_1}^{+\infty} 1 \cdot \phi(s) \, ds = 1 - \Phi(s_1).
  \]

Total \( p \):
  \[
  p = p_1 + p_2 = \Phi(s_1) - L \int_{s_0}^{s_1} \Phi(s) \, ds + 1 - \Phi(s_1) = 1 - L \int_{s_0}^{s_1} \Phi(s) \, ds.
  \]

\paragraph{Expression for \( I(y^*) \):}

The objective functional evaluated at \( y^*(s) \) is:
\[
I(y^*) = \int_{-\infty}^{+\infty} s y^*(s) \phi(s) \, ds.
\]

Again, split the integral:
\[
I(y^*) = \int_{s_0}^{s_1} s y^*(s) \phi(s) \, ds + \int_{s_1}^{+\infty} s y^*(s) \phi(s) \, ds.
\]

Compute each part:

  First integral (\( I_1 \)):
  \[
  I_1 = \int_{s_0}^{s_1} s [ L (s - s_0) ] \phi(s) \, ds = L \int_{s_0}^{s_1} s (s - s_0) \phi(s) \, ds.
  \]

  Expand \( s (s - s_0) = s^2 - s_0 s \):
  \[
  I_1 = L \left( \int_{s_0}^{s_1} s^2 \phi(s) \, ds - s_0 \int_{s_0}^{s_1} s \phi(s) \, ds \right).
  \]

  Recall standard integrals:
  \begin{align*}
  \int s^2 \phi(s) \, ds &= -s \phi(s) + \Phi(s), \\
  \int s \phi(s) \, ds &= -\phi(s).
  \end{align*}

  Evaluate the integrals:
  \begin{align*}
  \int_{s_0}^{s_1} s^2 \phi(s) \, ds &= [ -s \phi(s) + \Phi(s) ]_{s_0}^{s_1} = [ -s_1 \phi(s_1) + \Phi(s_1) ] - [ -s_0 \phi(s_0) + \Phi(s_0) ], \\
  \int_{s_0}^{s_1} s \phi(s) \, ds &= [ -\phi(s) ]_{s_0}^{s_1} = -\phi(s_1) + \phi(s_0).
  \end{align*}

  Substitute back into \( I_1 \):
  \begin{align*}
  I_1 &= L \left( [ -s_1 \phi(s_1) + \Phi(s_1) ] - [ -s_0 \phi(s_0) + \Phi(s_0) ] - s_0 ( -\phi(s_1) + \phi(s_0) ) \right) \\
      &= L \left( \Phi(s_1) - \Phi(s_0) - s_1 \phi(s_1) + s_0 \phi(s_0) + s_0 \phi(s_1) - s_0 \phi(s_0) \right) \\
      &= L \left( \Phi(s_1) - \Phi(s_0) - (s_1 - s_0) \phi(s_1) \right).
  \end{align*}

  Since \( s_1 - s_0 = \frac{1}{L} \):
  \[
  I_1 = L \left( \Phi(s_1) - \Phi(s_0) - \frac{1}{L} \phi(s_1) \right) = L [ \Phi(s_1) - \Phi(s_0) ] - \phi(s_1).
  \]

  Second integral (\( I_2 \)):
  \[
  I_2 = \int_{s_1}^{+\infty} s \cdot 1 \cdot \phi(s) \, ds.
  \]

  Recall that:
  \[
  \int_{s}^{+\infty} s \phi(s) \, ds = \phi(s).
  \]

  Therefore:
  \[
  I_2 = \phi(s_1).
  \]

  Total \( I(y^*) \):
  \[
  I(y^*) = I_1 + I_2 = \left( L [ \Phi(s_1) - \Phi(s_0) ] - \phi(s_1) \right) + \phi(s_1) = L [ \Phi(s_1) - \Phi(s_0) ].
  \]

Thus, the optimal value of the objective functional is:
\[
I(y^*) = L [ \Phi(s_1) - \Phi(s_0) ].
\]
\end{proof}

Now we can prove Theorem 2.

\begin{proof}

Let us assume \( \sigma = 1 \). We start by expressing the gradient of \( \phi^{-1}(\tilde{F}(\vx)) \):
\[
\nabla \phi^{-1}(\tilde{F}(\vx)) = \frac{\nabla \tilde{F}(\vx)}{\phi'(\phi^{-1}(\tilde{F}(\vx)))} \ .
\]

We aim to show that for any unit vector \( \vu \), the following inequality holds:
\[
\vu^\top \nabla \tilde{F}(\vx) \leq L(F) \left[ \Phi\left(s_0 + \frac{1}{L(F)}\right) - \Phi(s_0) \right],
\]
where \( s_0 \) is determined by:
\[
\tilde{F}(\vx) = 1 - L(F) \int_{s_0}^{s_0 + \frac{1}{L(F)}} \Phi(s) \, ds \ .
\]

\textbf{Applying Stein's Lemma}

Using Stein's lemma, we obtain the expression for \( \vu^\top \nabla \tilde{F}(\vx) \):
\[
\vu^\top \nabla \tilde{F}(\vx) = \mathbb{E}_{\delta \sim \mathcal{N}(0, \mI)} \left[ \vu^\top \delta F(\vx + \delta) \right],
\]
where \( \delta \) is a Gaussian vector with mean zero and identity covariance matrix. Our goal is to bound the maximum of this expression under the following constraints:

\begin{itemize}
    \item \( 0 \leq F(\vx) \leq 1 \)
    \item \( \mathbb{E}_{\delta \sim \mathcal{N}(0, \mI)} \left[ F(\vx + \delta) \right] = p \)
    \item \( F \) is Lipschitz continuous with a Lipschitz constant \( L(F) \).
\end{itemize}

Let \( y(z) = F(z + \vx) \). Then the problem can be recast as:
\[
(P) \quad \max_{||\vu|| = 1} \max_{y} \quad I(y) = \int_{\mathbb{R}^d} \vu^\top s \, y(s) \, \phi(s) \, ds
\]
subject to:
\[
L(y) \leq L(F), \quad 0 \leq y(s) \leq 1, \quad \int_{\mathbb{R}^d} y(s) \phi(s) \, ds = p.
\]

Without loss of generality, we can take \( \vu = (1, 0, \ldots, 0)^\top \), as Gaussian vectors are rotationally invariant and rotation preserves the \( \ell^2 \)-norm.

\textbf{Reducing to a One-Dimensional Problem}

Decompose \( \vs \) as \( \vs = (s_1, \vt) \), where \( s_1 = \vu^\top \vs \) and \( \vt = (\vs_2, \ldots, \vs_d) \). The density function \( \phi(\vs) \) can be factorized as:
\[
\phi(\vs) = \phi_1(s_1) \phi_{d-1}(\vt),
\]
where \( \phi_1 \) and \( \phi_{d-1} \) are the density functions of a standard Gaussian in one dimension and \( d-1 \) dimensions, respectively.

Rewrite the objective function as:
\[
I(y) = \int_{\mathbb{R}^{d-1}} \int_{\mathbb{R}} s_1 y(s_1, \vt) \phi_1(s_1) \phi_{d-1}(\vt) \, ds_1 \, d\vt.
\]

Define \( g(s_1) = \int_{\mathbb{R}^{d-1}} y(s_1, \vt) \phi_{d-1}(\vt) \, d\vt \). Then, the integral becomes:
\[
I(y) = I(g) = \int_{\mathbb{R}} s_1 g(s_1) \phi_1(s_1) \, ds_1.
\]

\textbf{Reformulating the Constraints for \( g(s_1) \)}

The integral constraint on \( p \) directly translates to:
\[
\int_{\mathbb{R}} g(s_1) \phi_1(s_1) \, ds_1 = p.
\]

Since \( 0 \leq y(\vs) \leq 1 \) for all \( \vs \in \mathbb{R}^d \), it follows that \( 0 \leq g(s_1) \leq 1 \) for all \( s_1 \in \mathbb{R} \).

For the Lipschitz condition, we have:
\[
|\nabla g(s_1)| \leq \int_{\mathbb{R}^{d-1}} |\nabla y(s_1, \vt)| \phi_{d-1}(\vt) \, d\vt \leq L(F).
\]

\textbf{Solving the One-Dimensional Problem}

The reduced problem is now:
\[
(Q) \quad \max_{g} \quad I(g) = \int_{\mathbb{R}} s_1 g(s_1) \phi_1(s_1) \, ds_1
\]
subject to:
\[
L(g) \leq L(F), \quad 0 \leq g(s_1) \leq 1, \quad \int_{\mathbb{R}} g(s_1) \phi_1(s_1) \, ds_1 = p.
\]

Applying Lemma~\ref{lemma:optimal_control_problem}, the optimal solution \( g^*(s_1) \) is given by:
\[
g^*(s_1) =
\begin{cases}
0, & s_1 \leq s_0, \\
L(F) (s_1 - s_0), & s_0 < s_1 < s_0 + \frac{1}{L(F)}, \\
1, & s_1 \geq s_0 + \frac{1}{L(F)}.
\end{cases}
\]

Here, \( s_0 \) is determined by the equation:
\begin{equation}
    \label{eq:compute_s0}
    \tilde{F}(\vx) = 1 - L(F) \int_{s_0}^{s_0 + \frac{1}{L(F)}} \Phi(s) \, ds.
\end{equation}

Evaluating \( I(g^*) \), we find:
\[
\vu^\top \nabla \tilde{F}(\vx) \leq L(F) \left[ \Phi\left(s_0 + \frac{1}{L(F)}\right) - \Phi(s_0) \right].
\]

For arbitrary \( \sigma \), the result follows by scaling the variables appropriately and performing a change of integration variable.

Finally,
\begin{align*}
    &L\left( \Phi^{-1} \circ \tilde{F}, B(\vx, \rho) \right) 
     = 
    &\sup_{\vx^\prime \in B(\vx, \rho)} \left\{ \frac{L(F) \left[ \Phi_\sigma\left(s_0(\vx^\prime) + \dfrac{1}{L(F)}\right) - \Phi_\sigma(s_0(\vx^\prime)) \right]}{
    \dfrac{1}{\sqrt{2\pi}} \exp\left( -\dfrac{1}{2} \left( \Phi^{-1}(\tilde{F}(\vx^\prime)) \right)^2 \right)
    } \right\} \
\end{align*}

\end{proof}

\begin{rmk}
The certified radius $ R_{\mathrm{monoLip}} $ is always larger than or equal to $ R_{\mathrm{mono}} $. This follows from the fact that the Lipschitz constant of the smoothed classifier depends on the value of the optimization problem $ I(g^*) $. A higher $ I(g^*) $ leads to a higher Lipschitz constant, which results in a smaller radius for the same margin.
\begin{proof}
The certified radius $ R $ is given by:
$$
R = \frac{\text{margin}}{\text{Lipschitz constant}}.
$$
For a fixed margin, a smaller Lipschitz constant increases the certified radius. The Lipschitz constant of the smoothed classifier $ \Phi^{-1} \circ \tilde{F} $ is:
$$
L = \sup_{x' \in B(x, \rho)} 
\sup_{ \norm{\vu} = 1}
\frac{ \vu^\top \nabla \tilde{F}(x')}{\phi\left(\Phi^{-1}(\tilde{F}(x'))\right)}.
$$
Here, the numerator $ \vu^\top \nabla \tilde{F}(x') $ is bounded by the value $ I(g^*) $ obtained from the optimization problem. For $ R_{\mathrm{mono}} $, the optimization problem $ (Q) $ is unconstrained, leading to a step function $ g^*(s_1) $ that maximizes $ I(g^*) $. For $ R_{\mathrm{monoLip}} $, the optimization problem $ (Q') $ includes a Lipschitz constraint $ L(g) \leq L(F) $, resulting in a smoother solution $ g^*(s_1) $ with a lower $ I(g^*) $.

Since $ I(g^*)_{Q'} \leq I(g^*)_{Q} $, it follows that:
$$
L_{\mathrm{monoLip}} \leq L_{\mathrm{mono}}.
$$

Thus, the certified radius satisfies:
$$
R_{\mathrm{monoLip}} \geq R_{\mathrm{mono}}.
$$

Even when considering the Lipschitz constant as the supremum over a ball around $ x $, the additional constraint in $ (Q') $ ensures that $ I(g^*) $ is smaller within the ball, making $ L_{\mathrm{monoLip}} $ consistently lower than $ L_{\mathrm{mono}} $. Therefore, the inequality $ R_{\mathrm{monoLip}} \geq R_{\mathrm{mono}} $ holds universally.
\end{proof}
\end{rmk}

\subsection{Computation of Product Upper Bound (PUB) for ResNet Architectures}

The Product Upper Bound (PUB) for the Lipschitz constant of a ResNet model is computed as the product of the Lipschitz constants of its individual layers. Below, we summarize how the Lipschitz constant is estimated for different components in the ResNet architecture.

For convolutional layers (Conv2D), the Lipschitz constant is the spectral norm of the layer, which corresponds to the largest singular value of the weight matrix. This is computed via power iteration applied directly to the convolution operator as in \citet{ryu2019plug}. The iteration alternates between forward and transposed convolutions until convergence.

For fully connected (dense) layers, the Lipschitz constant is also the spectral norm of the weight matrix. Power iteration is similarly applied, where random vectors are iteratively multiplied by the weight matrix and transpose to approximate the largest singular value.

Batch normalization layers (BatchNorm) are treated as affine transformations during inference. The Lipschitz constant is computed as:
\[
L_{\text{BN}} = \max_{i} \left| \frac{\gamma_i}{\sqrt{\sigma_i^2 + \epsilon}} \right|,
\]
where \( \gamma_i \) is the scale parameter, \( \sigma_i^2 \) is the running variance, and \( \epsilon \) is a small stability constant. This ensures the Lipschitz constant accurately reflects the effect of batch normalization during evaluation.

Pooling layers, including max pooling and average pooling, are conservatively assumed to have a Lipschitz constant of 1. Max pooling performs a maximum operation over the input, which does not amplify the norm, and average pooling is a linear operation with bounded norm.

Residual connections involve summing the output of a subnetwork with its input. The Lipschitz constant of a residual block is given by:
\[
L_{\text{residual}} \leq L_{\text{main}} + 1,
\]
where \( L_{\text{main}} \) is the Lipschitz constant of the main path (e.g., convolutional layers). The identity mapping has a Lipschitz constant of 1, and the addition operation preserves the norm.

The overall PUB for a ResNet model is computed by multiplying the Lipschitz constants of all layers, taking into account the components described above. This product provides a conservative upper bound for the Lipschitz constant of the network, which is crucial for robust certification and stability analysis.

\subsection{Explanation of PUB Overestimation with Linear Network Example}

The Product Upper Bound (PUB) overestimates the Lipschitz constant because it assumes perfect alignment of singular vectors across layers, which rarely happens in practice. For a linear network with 110 fully connected layers, where each layer has a spectral norm of 2, PUB computes the Lipschitz constant as $2^{110}$, an astronomically large number. However, the true Lipschitz constant is the spectral norm of the product of all weight matrices, which can be significantly smaller due to cancellations and the lack of perfect alignment.

This discrepancy becomes even more pronounced in deep nonlinear networks like ResNets. Nonlinearities such as ReLU further reduce the Lipschitz constant by "cutting" the norm of gradients or activations approximately in half, on average. Additionally, residual connections and batch normalization contribute to the reduction of the true Lipschitz constant. PUB’s exponential growth with depth makes it an overly conservative measure, especially for robustness certification.

We initialized a small linear network with depth 110, input dimension 100, and output dimension 10, using Kaiming Uniform Initialization~\citep{he2015delving}. In this setup, the true Lipschitz constant at initialization was found to be extremely low, effectively around numerical zero, while the PUB was significantly higher, approximately $2.6 \times 10^{5}$. This stark difference illustrates the conservative nature of PUB and its limitations in accurately reflecting the true Lipschitz behavior of deep networks.

\subsection{Remark on use of Brent's solver}
To compute the term $s_0$ in Equation~(\ref{eq:compute_s0}), we use Brent's solver.
Numerical solvers like Brent's method achieve precision comparable to CDF estimation (~$10^{-16}$ for double precision) and are well-behaved if they do converge, delivering machine-level accuracy. 
 If convergence fails we can revert to the default baseline radius, ensuring robust confidence intervals without type-I errors.



\end{document}